\documentclass[conference,10pt]{IEEEtran}
\IEEEoverridecommandlockouts
\usepackage{amsmath,amsfonts,amsthm,amssymb,bbm}
\usepackage{cite}

\usepackage{balance}
\usepackage{mathrsfs}
\usepackage{xcolor}
\usepackage{tikz}
\usepackage{pgfplots}
\usepackage{colortbl}
\usetikzlibrary{positioning,quotes,angles,patterns,intersections,pgfplots.fillbetween,calc}
\usepackage{tkz-euclide}
\usepackage{mathtools}
\usepackage{url}
\usetikzlibrary{decorations.pathmorphing}
\usepackage{enumitem}
\usepackage{stfloats}
\pgfplotsset{compat=1.17} 
\usepackage[normalem]{ulem}
\theoremstyle{plain} 
\newtheorem{theorem}{Theorem}

\newtheorem{definition}{Definition}
\newtheorem{lemma}{Lemma}

\theoremstyle{definition} \newtheorem{remark}{Remark}
\theoremstyle{definition} 
\usetikzlibrary{calc,shapes.geometric}

\title{$\alpha$-GAN: Convergence and Estimation Guarantees}
\author{%
Gowtham R. Kurri, Monica Welfert, Tyler Sypherd, and Lalitha Sankar \\
 Arizona State University, \texttt{\{gkurri,mwelfert,tsypherd,lalithasankar\}@asu.edu} 
 \thanks{This work is supported in part by NSF grants CIF-1901243, CIF-1815361, CIF-2007688, CIF-2134256, CIF-2031799, and CIF-1934766.}
}
\def \extended {1} 
\begin{document}
\maketitle
\begin{abstract} 
We prove a two-way correspondence between the min-max optimization of general CPE loss function GANs and the minimization of associated $f$-divergences. We then focus on $\alpha$-GAN, defined via the $\alpha$-loss, which interpolates several GANs (Hellinger, vanilla, Total Variation) and corresponds to the minimization of the Arimoto divergence. We show that the Arimoto divergences induced by $\alpha$-GAN equivalently converge, for all $\alpha\in \mathbb{R}_{>0}\cup\{\infty\}$. However, under restricted learning models and finite samples, we provide estimation bounds which indicate diverse GAN behavior as a function of $\alpha$. Finally, we present empirical results on a toy dataset that highlight the practical utility of tuning the $\alpha$ hyperparameter. 
\end{abstract}

\section{Introduction}
%\noteTS{A high-level comment: in our ITW paper, we told a good story about GANs, the state-of-the-art, and the challenges they face. Perhaps we can include some of those words in here, while still having a very mathematical introduction. I imagine the IT community might not know too much about GANs. In total, maybe one more small paragraph, or sentences interspersed.}
%\noteLS{I think it's okay since most folks do know about GANs; if space permits, we can add more later. But I agree on adding notes on the CHALLLENGES they face as that is what we wish to address.}
Generative adversarial networks (GANs) are \emph{generative models} capable of producing new samples from an unknown (real) distribution using a finite number of training data samples. A GAN is composed of two modules, a generator $G$ and a discriminator $D$, parameterized by vectors $\theta\in\Theta\subset \mathbb{R}^{n_g}$ and $\omega\in\Omega\subset\mathbb{R}^{n_d}$, respectively, which play an adversarial game with one another. The generator $G_\theta$ takes as input noise $Z\sim P_Z$ and maps it to a data sample in $\mathcal{X}$ via the mapping $z\mapsto G_\theta(z)$ with an aim of mimicking data from the real distribution $P_{r}$. For an input $x\in\mathcal{X}$, the discriminator classifies if it is real data or generated data by outputting $D_\omega(x)\in[0,1]$, the probability that $x$ comes from $P_r$ (real) as opposed to $P_{G_\theta}$ (synthetic). The opposing goals of the generator and the discriminator lead to a zero-sum min-max game with a chosen value function $V(\theta,\omega)$ resulting in an optimization problem given by
\thinmuskip=1mu
\begin{align}\label{eqn:GANgeneral}
    \inf_{\theta\in\Theta}\sup_{\omega\in\Omega} V(\theta,\omega). 
\end{align}
Goodfellow \emph{et al.}~\cite{Goodfellow14} introduced GANs via a value function
% \thinmuskip=0mu
% \thickmuskip=0mu
% \medmuskip=0mu
% \begin{align}
%     V_\text{VG}(\theta,\omega)
%     =\mathbb{E}_{X\sim P_r}[\log{D_\omega(X)}]+\mathbb{E}_{X\sim P_{G_\theta}}[\log{(1-D_\omega(X))}],\label{eq:Goodfellowobj}
% \end{align}
% \begin{align}
%     &V_\text{VG}(\theta,\omega) \nonumber \\
%     &=\mathbb{E}_{X\sim P_r}[\log{D_\omega(X)}]+\mathbb{E}_{X\sim P_{G_\theta}}[\log{(1-D_\omega(X))}],\label{eq:Goodfellowobj}
% \end{align}
\begin{multline}
V_\text{VG}(\theta,\omega) \\
=\mathbb{E}_{X\sim P_r}[\log{D_\omega(X)}]+\mathbb{E}_{X\sim P_{G_\theta}}[\log{(1-D_\omega(X))}],\label{eq:Goodfellowobj}
\end{multline}
for which they showed that, when the discriminator class $\{D_\omega\}_{\omega\in\Omega}$ is rich enough, \eqref{eqn:GANgeneral} simplifies to $\inf_{\theta\in\Theta} 2D_{\text{JS}}(P_r||P_{G_\theta})-\log{4}$, where $D_{\text{JS}}(P_r||P_{G_\theta})$ is the Jensen-Shannon divergence~\cite{Lin91} between $P_r$ and $P_{G_\theta}$. This simplification is achieved, for any $G_\theta$, by the discriminator $D_{\omega^*}(x)$ maximizing \eqref{eq:Goodfellowobj} which has the form
    \begin{align}
    D_{\omega^*}(x)=\frac{p_r(x)}{p_r(x)+p_{G_\theta}(x)},
    \end{align}
where $p_r$ and $p_{G_\theta}$ are the corresponding densities of the distributions $P_r$ and $P_{G_\theta}$, respectively, with respect to a base measure $dx$ (e.g., Lebesgue measure). 

Various other GANs have been studied in the literature (e.g., $f$-divergence based GANs known as $f$-GAN~\cite{NowozinCT16}, IPM based GANs~\cite{ArjovskyCB17,sriperumbudur2012empirical,liang2018well}, Cumulant GAN~\cite{pantazis2020cumulant}, R\'{e}nyiGAN~\cite{bhatia2021least}, to name a few) with different value functions. In each case, the corresponding min-max optimization problem simplifies to minimizing some measure of divergence between the real and generated distributions. Yet, a methodical way to compare and operationally interpret GAN value functions remains open.
%the problem of how to choose a divergence to minimize for which one can also ascribe an operational meaning remains open. % whether it alleviates common problems associated with GANs remains.  % \noteLS{While a variety of GANs minimizing different $f$-divergences have been proposed, our work demonstrates that a loss-function approach to the GAN objective function offers an operational interpretation of the divergence (e.g., Arimoto divergence for $\alpha$-loss) and allows meaningful comparisons of different divergence-minimizing GANs.} 
 %\noteLS{Indeed a good place to insert issues with GANs and the need to study them in a principled fashion.}

Recently, in~\cite{KurriSS21}, we introduced a loss function~\cite{reid2010composite} perspective of GANs where we show that a GAN can be formulated using \emph{any} class probability estimation (CPE) loss $\ell(y,\hat{y})$ with inputs $y\in\{0,1\}$ (the true label) and predictor $\hat{y}\in[0,1]$ (soft prediction of $y$). 
%\noteTS{we need to define what a symmetric CPE loss is for later (Section 2). a symmetric CPE loss is a CPE loss where $\ell(1,\hat{y}) = \ell(0,1-\hat{y})$. Also, let's cite~\cite{reid2010composite} to cover our bases.}
We show that using CPEs, the value function (objective) in \eqref{eqn:GANgeneral} can be written as
% \thinmuskip=0mu
% \thickmuskip=0mu
% \medmuskip=0mu
\begin{multline}
   V(\theta,\omega) \\ =\mathbb{E}_{X\sim P_r}[-\ell(1,D_\omega(X))]+\mathbb{E}_{X\sim P_{G_\theta}}[-\ell(0,D_\omega(X))]\label{eqn:lossfnps1}
\end{multline}
% \begin{align}
%   &V(\theta,\omega) \nonumber \\
%   &=\mathbb{E}_{X\sim P_r}[-\ell(1,D_\omega(X))]+\mathbb{E}_{X\sim P_{G_\theta}}[-\ell(0,D_\omega(X))]\label{eqn:lossfnps1}
% \end{align}
(see \if \extended 0%
an extended version~\cite[Appendix~A]{ISITextenver}
\fi%
\if \extended 1%
Appendix~\ref{appendix:details-omitted-lossfn}
\fi%
for more details on this). We specialize the setup in \eqref{eqn:lossfnps1} to introduce $\alpha$-GAN using $\alpha$-loss, a tunable loss function parameterized by $\alpha\in\mathbb{R}_{>0}\cup\{\infty\}$~\cite{sypherd2019tunable,sypherd2021journal}, and with the loss function 
\begin{align} \label{eq:cpealphaloss}
      \ell_\alpha(y,\hat{y})\coloneqq\frac{\alpha}{\alpha-1}\left(1-y\hat{y}^{\frac{\alpha-1}{\alpha}}-(1-y)(1-\hat{y})^{\frac{\alpha-1}{\alpha}}\right).
\end{align}
In~\cite{KurriSS21}, we show that the $\alpha$-GAN formulation allows interpolating between various $f$-divergence based GANs including the Hellinger GAN~\cite{NowozinCT16} ($\alpha=1/2$), the vanilla GAN~\cite{Goodfellow14} ($\alpha=1$), and the Total Variation (TV) GAN~\cite{NowozinCT16} ($\alpha=\infty$), as well as IPM based GANs including WGAN~\cite{ArjovskyCB17} (for $\alpha=\infty$ and an appropriately constrained discriminator class). 
We also show that, for large enough discriminator capacity, the min-max optimization problem for $\alpha$-GAN in \eqref{eqn:GANgeneral} simplifies to 
\begin{align}\label{eqn:inf-obj-alpha}
    \inf_{\theta\in\Theta} D_{f_\alpha}(P_r||P_{G_\theta})+\frac{\alpha}{\alpha-1}\left(2^{\frac{1}{\alpha}}-2\right),
\end{align}
where $D_{f_\alpha}(P_r||P_{G_\theta})$ is the Arimoto divergence~\cite{osterreicher1996class,LieseV06} given by
\begin{align}\label{eqn:alpha-divergence}
D_{f_\alpha}(P||Q)=\frac{\alpha}{\alpha-1}\left(\int_\mathcal{X} \left(p(x)^\alpha+q(x)^\alpha\right)^\frac{1}{\alpha} dx-2^{\frac{1}{\alpha}}\right).
\end{align}
This results for the $D_{\omega^*}(x)$ maximizing~\eqref{eqn:lossfnps1} with
\begin{align}\label{eqn:optimaldoisc}
    D_{\omega^*}(x)=\frac{p_r(x)^\alpha}{p_r(x)^\alpha+p_{G_\theta}(x)^\alpha}.
\end{align}
We build on \cite{KurriSS21} to investigate various aspects of CPE loss-based GANs including $\alpha$-GAN as summarized below:
%We now build upon this prior work~\cite{KurriSS21} and investigate various stability aspects of $\alpha$-GANs (and also any CPE loss function based GAN) during training, including experimental results, in an effort to address the key challenges in designing GANs. %\noteTS{phrasing in previous sentence is a bit awkward.} 
%Our contributions are as follows:
\begin{itemize}[leftmargin=*]
    \item We first establish a two-way correspondence between CPE loss function-based GANs and $f$-divergences building upon a correspondence between margin-based loss functions and $f$-divergences~\cite{NguyenWJ09} (Theorem~\ref{thm:correspondence}). % In particular, we show that every \emph{symmetric} CPE loss function based GAN corresponds to minimizing an $f$-divergence. Conversely, for every GAN minimizing a \emph{symmetric} $f$-divergence, there exists a corresponding (symmetric) CPE loss function based GAN (Theorem~\ref{thm:correspondence}). 
    This not only complements the connection established between the variational form of $f$-divergence in~\cite{NguyenWJ10} and the $f$-GAN formulation in~\cite{NowozinCT16} but, more crucially, also provides an easier way to implement a variety of $f$-GANs in practice. 
    \item %We study \emph{convergence} properties of $\alpha$-GAN in the presence of sufficiently large number of samples and discriminator capacity. 
    For a sufficiently large number of samples and ample discriminator capacity, we show that Arimoto divergences for all $\alpha\in\mathbb{R}_{>0}\cup\{\infty\}$ are \emph{equivalent} in convergence (Theorem~\ref{thm:equivalenceinconvergence}). This generalizes such an equivalence known~\cite{ArjovskyCB17,liu2017approximation} only for special cases, i.e., Jensen-Shannon divergence (JSD) for $\alpha=1$, squared Hellinger distance for $\alpha=1/2$, and total variation distance (TVD) for $\alpha=\infty$, thus providing a unified perspective on the convergence guarantees of several existing GANs. %Our proof techniques also indicate a 
    We present a simpler proof of the equivalence between JSD and TVD~\cite[Theorem~2(1)]{ArjovskyCB17}.
    \item When the generator and the discriminator models are neural networks of limited capacity, %Ji~\emph{et al.}~\cite{JiZL21} obtained bounds on the estimation error for training GANs based on the neural net distance introduced by Arora \emph{et al.}~\cite{AroraGLMZ17}. Building on \cite{JiZL21}, 
    we present bounds on the estimation error for CPE loss GANs (including $\alpha$-GAN) %training Lipschitz loss function based GANs 
    by leveraging a contraction lemma on Rademacher complexity~\cite[Lemma~26.9]{shalev2014understanding} (Theorem~3). 
    \item Finally, we highlight the value of tuning $\alpha$ to generate distribution-accurate synthetic data for a toy dataset. 
    % \noteTS{we need one or two more lines here because experiments are about 1 page total, and we have good results}
\end{itemize}
% However, in practice, it is well known that all the existing GANs suffer from a number of instability problems that arise during training~\cite{huszar2015not,metz2016unrolled,salimans2016improved,arjovsky2017towards,GulrajaniAADC17,wiatrak2019stabilizing}. Hence, there is a quest for a proper value function with the objective of improving the performance of GANs.

\section{Main Results}\label{section:mainresults}
\if \extended 0%
We now present our three main results; for space reasons, we collate all our proofs in an extended version~\cite{ISITextenver}. 
\fi%
\if \extended 1%
We now present our three main results here.
\fi%

\subsection{Correspondence: CPE loss GANs and $f$-divergences}
We first establish a precise correspondence between the family of GANs based on CPE loss functions and a family of $f$-divergences. 
% This connection generalizes a previously given correspondence between the GAN based on $\alpha$-loss (i.e., $\alpha$-GAN) and the Arimoto divergence~\cite{KurriSS21}.
We do this by building upon a relationship between margin-based loss functions~\cite{BartlettJM06} and $f$-divergences first demonstrated by Nguyen \emph{et al.}~\cite{NguyenWJ09} and leveraging our CPE loss function perspective of GANs given in~\eqref{eqn:lossfnps1}. 
This complements the connection established  by Nowozin \emph{et al.}~\cite{NowozinCT16} between the variational estimation approach of $f$-divergences~\cite{NguyenWJ10} and $f$-divergence based GANs.
 We call a CPE loss function $\ell(y,\hat{y})$ \emph{symmetric}~\cite{reid2010composite} if $\ell(1,\hat{y})=\ell(0,1-\hat{y})$ and an $f$-divergence $D_f(\cdot\|\cdot)$ \emph{symmetric}~\cite{liese1987convex,sason2015tight} if $D_f(P\|Q)=D_f(Q\|P)$. 
We assume GANs with sufficiently large number of samples and ample discriminator capacity.
\begin{theorem}\label{thm:correspondence}
For any symmetric CPE loss GAN with a value function  in~\eqref{eqn:lossfnps1}, the min-max optimization in \eqref{eqn:GANgeneral} reduces to minimizing an $f$-divergence. Conversely, for any GAN designed to minimize a symmetric $f$-divergence, there exists a (symmetric) CPE loss GAN minimizing the same $f$-divergence. 
\end{theorem}
\begin{proof}[Proof sketch]\let\qed\relax
Let $\ell$ be the symmetric CPE loss of a given CPE loss GAN; note that $\ell$ has a bivariate input $(y,\hat{y})$ (\emph{e.g.} in~\eqref{eq:cpealphaloss}), where $y \in \{0,1\}$ and $\hat{y} \in [0,1]$.
We define an associated margin-based loss function $\tilde{\ell}$ using a bijective link function (satisfying a mild regularity condition); note that a margin-based loss function has a univariate input $z \in \mathbb{R}$ (\emph{e.g.}, the logistic loss $\tilde{l}^{\text{log}}(z) = \log{(1+e^{-z})}$) and the bijective link function maps $z \rightarrow \hat{y}$ (see~\cite{BartlettJM06,reid2010composite} for more details). 
We show after some manipulations that the inner optimization of the CPE loss GAN reduces to an $f$-divergence with
\begin{align}\label{eqn:thm1proofsketch1}
f(u):=-\inf_{t\in\mathbb{R}}\left(\tilde{\ell}(-t)+u\tilde{\ell}(t)\right).
\end{align}
For the converse, given a symmetric $f$-divergence, using \cite[Corollary~3 and Theorem~1(b)]{NguyenWJ09}, note that there exists a margin-based loss  $\tilde{\ell}$ such that \eqref{eqn:thm1proofsketch1} holds. The rest of the argument follows from defining a symmetric CPE loss $\ell$ from this margin-based loss $\tilde{\ell}$ via the \textit{inverse} of the same link function. See
\if \extended 1%
Appendix~\ref{apndx:proof-of-thm1} 
\fi%
\if \extended 0%
\cite[Appendix~C]{ISITextenver}
\fi%
for the detailed proof.
\end{proof}
%\noteTS{Perhaps we can give some intuition for what this result means}
We note that this connection in Theorem~\ref{thm:correspondence} generalizes a previously given correspondence between $\alpha$-GAN and the Arimoto divergence~\cite{KurriSS21}. A consequence of Theorem \ref{thm:correspondence} is that it offers an interpretable way to design GANs and connect a desired measure of divergence to a corresponding loss function, where the latter is easier to implement in practice. Moreover, CPE loss based GANs, including $\alpha$-GAN, inherit the intuitive and compelling interpretation of vanilla GANs that the discriminator should assign higher likelihood values to real samples and lower ones to generated samples 
(see \if \extended 0%
\cite[Appendix~A]{ISITextenver}).
\fi%
\if \extended 1%
Appendix~A).
\fi%

\subsection{Convergence Properties of $\alpha$-GAN}
Building on the above one-to-one correspondence, % between CPE loss based GANs and $f$-divergences,
%Having established a one-to-one correspondence between CPE loss based GANs and $f$-divergences, %we now establish convergence guarantees on the divergence for a specific CPE GAN, $\alpha$-GAN.  
we now present \emph{convergence} results for a specific CPE loss based GAN, namely $\alpha$-GAN, thereby providing a unified perspective on the convergence of a variety of $f$-divergences that arise when optimizing GANs. Here again, we assume a sufficiently large number of samples and ample discriminator capacity.
In~\cite{liu2017approximation}, Liu~\emph{et al.} address the following question in the context of convergence analysis of any GAN: 
%(that the generator wants to minimize)
For a sequence of generated distributions $(P_n)$, does convergence of a divergence between the generated distribution $P_{n}$ and a fixed real distribution $P$ to the global minimum lead to some standard notion of distributional convergence of $P_n$ to $P$? They answer this question in the affirmative provided the sample space $\mathcal{X}$ is a compact metric space.

Liu \emph{et al.}~\cite{liu2017approximation} formally define any divergence that results from the inner optimization of a general GAN in~\eqref{eqn:GANgeneral} as an \emph{adversarial divergence}~\cite[Definition~1]{liu2017approximation}, thus broadly capturing the divergences used by a number of existing GANs, including vanilla GAN~\cite{Goodfellow14}, $f$-GAN~\cite{NowozinCT16}, WGAN~\cite{ArjovskyCB17}, and MMD-GAN~\cite{dziugaite2015training}.
Indeed, the divergence that results from the inner optimization of CPE loss function GAN~\eqref{eqn:lossfnps1} (including $\alpha$-GAN) is also an adversarial divergence.
For \emph{strict adversarial divergences} (a subclass of the adversarial divergences where the minimizer of the divergence is uniquely the real distribution), %~\cite[Definition~3]{liu2017approximation}, 
Liu~\emph{et al.}~\cite{liu2017approximation} show that convergence of the divergence to its global minimum implies weak convergence of the generated distribution to the real distribution. Interestingly, this also leads to a structural result on the class of strict adversarial divergences~\cite[Figure~1 and Corollary~12]{liu2017approximation} based on a notion of \emph{relative strength} between adversarial divergences. 
We note that the Arimoto divergence $D_{f_{\alpha}}$ in~\eqref{eqn:alpha-divergence} is a strict adversarial divergence.
We briefly summarize the following terminology from Liu \emph{et al.}~\cite{liu2017approximation} to present our results on convergence properties of $\alpha$-GAN. Let $\mathcal{P}(\mathcal{X})$ be the probability simplex of distributions over $\mathcal{X}$.
\begin{definition}[Definition~11,\cite{liu2017approximation}] \label{def:equivalenceadversarialdivergence}
A {strict adversarial divergence} $\tau_1$ is said to be stronger than another strict adversarial divergence $\tau_2$ (or $\tau_2$ is said to be weaker than $\tau_1$) if for any sequence of probability distributions $(P_n)$ and target distribution $P$ (both in $\mathcal{P}(\mathcal{X})$), $\tau_1(P\|P_n)\rightarrow 0$ as $n\rightarrow \infty$ implies $\tau_2(P\|P_n)\rightarrow 0$ as $n\rightarrow \infty$. We say $\tau_1$ is equivalent to $\tau_2$ if $\tau_1$ is both stronger and weaker than $\tau_2$.
\end{definition}

Arjovsky \emph{et al.}~\cite{ArjovskyCB17} proved that the Jensen-Shannon divergence (JSD) is equivalent to the total variation distance (TVD). 
Later, Liu \emph{et al.} showed that the squared Hellinger distance is equivalent to both of these divergences, meaning that all three divergences belong to the same equivalence class (see \cite[Figure~1]{liu2017approximation}). Noticing that the squared Hellinger distance, JSD, and TVD correspond to Arimoto divergences $D_{f_\alpha}(\cdot||\cdot)$ for $\alpha=1/2$, $\alpha=1$, and $\alpha=\infty$, respectively, it is natural to ask the question: Are Arimoto divergences for all $\alpha>0$ equivalent? We answer this question in the affirmative in Theorem~\ref{thm:equivalenceinconvergence}, thereby adding the Arimoto divergences for all other $\alpha \in \mathbb{R}_{>0}\cup\{\infty\}$ to the same equivalence class.
\begin{theorem}\label{thm:equivalenceinconvergence}
The Arimoto divergences for all $\alpha\in\mathbb{R}_{>0}\cup\{\infty\}$ are equivalent in the sense of Definition~\ref{def:equivalenceadversarialdivergence}. That is, for a sequence of probability distributions $(P_n) \in \mathcal{P}(\mathcal{X})$ and a fixed distribution $P \in \mathcal{P}(\mathcal{X})$, $D_{f_{\alpha_1}}(P_n||P)\rightarrow 0$ as $n\rightarrow \infty$ if and only if $D_{f_{\alpha_2}}(P_n||P)\rightarrow 0$ as $n\rightarrow \infty$, for any $\alpha_1\neq \alpha_2$. 
\end{theorem}
\begin{remark}
We note that the proof techniques used in proving Theorem~\ref{thm:equivalenceinconvergence} give rise to a conceptually simpler proof of equivalence between JSD ($\alpha = 1$) and TVD ($\alpha = \infty$) proved earlier by Arjovsky \emph{et al.}~\cite[Theorem~2(1)]{ArjovskyCB17}, where measure-theoretic analysis was used. In particular, our proof of equivalence relies on the fact that TVD upper bounds JSD~\cite[Theorem~3]{Lin91}. See 
\if \extended 0%
\cite[Appendix~B]{ISITextenver}
\fi%
\if \extended 1%
Appendix~B
\fi%
for details.
\end{remark}
\begin{proof}[Proof sketch]\let\qed\relax
Noticing that $D_{f_\infty}(\cdot\|\cdot)$ is equal to TVD, denoted $D_{\text{TV}}(\cdot\|\cdot)$ (see \cite{osterreicher2003new}, \cite[Theorem~2]{KurriSS21}), it suffices to show that $D_{f_\alpha}(\cdot\|\cdot)$ is equivalent to $D_{\text{TV}}(\cdot\|\cdot)$, for $\alpha>0$. To show this, we employ an elegant result by \"{O}sterreicher and Vajda~\cite[Theorem~2]{osterreicher2003new} (with application in statistics) which gives lower and upper bounds on the Arimoto divergence in terms of TVD as
\begin{align}\label{eqn:boundsonArimotomain}
    \gamma_\alpha(D_{\text{TV}}(P||Q))\leq D_{f_\alpha}(P||Q)\leq \gamma_\alpha(1)D_{\text{TV}}(P||Q),
\end{align}
for an appropriately defined well-behaved (continuous, invertible, and bounded) function $\gamma_\alpha:[0,1]\rightarrow \mathbb{R}$. %\noteLS{GK, we should clarify that $\psi$ is a well-behaved bounded function in $\alpha$ and so the convergence is actually quite tight.} \noteGK{Yes, right. I added the main properties required now.} 
%\noteTS{GK, is there notation overload with $\psi$? Used in this proof sketch and also for $-\ell(0,\cdot)$?}
We use the lower and upper bounds in \eqref{eqn:boundsonArimotomain} to show that $D_{f_\alpha}(\cdot\|\cdot)$ is stronger than $D_{\text{TV}}(\cdot\|\cdot)$, and $D_{f_\alpha}(\cdot\|\cdot)$ is weaker than $D_{\text{TV}}(\cdot\|\cdot)$, respectively. Proof details are in %See 
\if \extended 1%
Appendix~\ref{proofoftheorem4} 
\fi%
\if \extended 0%
\cite[Appendix~D]{ISITextenver}. %the extended version~\cite[Appendix~D]{ISITextenver}
\fi%
%for a detailed proof.
\end{proof}
%\noteLS{GK, do you like this transition?} \noteGK{Yes, this looks good to me. Can we try to have a similar transition from subsection A to B also? (where we may possibly convey something along the lines convergence properties studied provide a unified perspective of convergence guarantees of various existing GANs [f-GANs] , etc..)}
Theorems \ref{thm:correspondence} and \ref{thm:equivalenceinconvergence} hold in the ideal setting of sufficient samples and discriminator capacity. In practice, however, GAN training is limited by both the number of training samples as well as the choice of $G_\theta$ and $D_\omega$. In fact, recent results by Arora \textit{et al.} \cite{AroraGLMZ17} show that under such limitations, convergence in divergence does not imply convergence in distribution, and have led to new metrics for evaluating GANs. We now study one such quantity, namely estimation error.

%\noteTS{reminder to talk about how convergence will be different in practice (cite experiments and estimation theory). This will give nice transition to estimation error.}
\subsection{Estimation Error Bounds for CPE Loss based GAN}
%So far, we have assumed sufficiently large number of samples and ample discriminator capacity. However, in practice, we only have limited number of 
We now consider a setting where we have a limited number of training samples\footnote{In practice, once a model is learned, one can generate any number of noise, and hence, synthetic samples; however, the number of real samples is the (finite sample) bottleneck for the goodness of the learned $G_\theta$ model.} $S_x=\{X_1,\dots,X_n\}$ and $S_z=\{Z_1,\dots,Z_m\}$ from $P_r$ and $P_Z$, respectively. 
Also, the discriminator and generator classes are typically neural networks; these limitations lead to estimation errors in training GANs~\cite{zhang2017discrimination,liang2018well,JiZL21}.  %\noteTS{what do you think of this - ``although, there is recent work investigating GANs in decision trees [cite] https://arxiv.org/abs/2201.11205.''}. 
While~\cite{JiZL21} models the interplay between both the discriminator and generator in the estimation error bounds, those developed in~\cite{zhang2017discrimination,liang2018well} do not explicitly capture the role of the generator. % the generator is not explicitly captured in the. 
We adopt the approach in~\cite{JiZL21}; to this end, %wherein they quantify estimation error 
we begin with the notion of neural net ($nn$) distance (first introduced in~\cite{AroraGLMZ17}) as defined for the setup in~\cite{ji2018minimax,JiZL21}:
% \thinmuskip=0mu
% \thickmuskip=0mu
% \medmuskip=0mu
% \begin{align}
%     d_{\mathcal{F}_{nn}}(P_r,P_{G_\theta})
%     &=\sup_{\omega\in\Omega}\left (\mathbb{E}_{X\sim P_r}\left[f_\omega(X)\right]-\mathbb{E}_{X\sim P_{G_\theta}}\left[f_\omega(X)\right] \right),
% \end{align}
\begin{align}
    &d_{\mathcal{F}_{nn}}(P_r,P_{G_\theta}) \nonumber \\
    &=\sup_{\omega\in\Omega}\left (\mathbb{E}_{X\sim P_r}\left[f_\omega(X)\right]-\mathbb{E}_{X\sim P_{G_\theta}}\left[f_\omega(X)\right] \right),
\end{align}
where the discriminator\footnote{In~\cite{JiZL21}, $f_{\omega}$ indicates a discriminator function that takes values in $\mathbb{R}$.} and generator $f_\omega(\cdot)$ and $G_\theta(\cdot)$, respectively, are neural networks. %~\cite[Equations~(7) and (8)]{JiZL21}.Leveraging the methodology in \cite{JiZL21}, 
We now introduce a loss-inclusive $d^{(\ell)}_{\mathcal{F}_{nn}}$ for CPE loss GANs (including $\alpha$-GAN) to highlight the effect of the \emph{loss} on the error. 
We begin with the following minimization for GAN training:
%We now define and quantify the estimation error in training CPE loss GANs (including $\alpha$-GAN), thereby highlighting the effect of the \emph{loss} on the error. We begin with the following minimization for GAN training:
\begin{align}\label{eqn:training-empirical}
    \inf_{\theta\in\Theta}d^{(\ell)}_{\mathcal{F}_{nn}}(\hat{P}_r,\hat{P}_{G_\theta}),
\end{align}
where $\hat{P}_r$ and $\hat{P}_{G_\theta}$ are the empirical real and generated distributions estimated from $S_x$ and $S_z$, respectively, and 
\begin{align}
    &d^{(\ell)}_{\mathcal{F}_{nn}}(\hat{P}_r,\hat{P}_{G_\theta})\nonumber\\
    &=\sup_{\omega\in\Omega}\left(\mathbb{E}_{X\sim \hat{P}_{r}}\phi \big(D_\omega(X) \big)+\mathbb{E}_{X\sim \hat{P}_{G_\theta}}\psi \big(D_\omega(X) \big)\right),
\end{align}
where for brevity we henceforth use $\phi(\cdot)\coloneqq -\ell(1,\cdot)$ and $\psi(\cdot)\coloneqq -\ell(0,\cdot)$. For $x\in\mathcal{X}\coloneqq\{x\in\mathbb{R}^d:||x||_2\leq B_x\}$ and  $z\in\mathcal{Z}\coloneqq\{z\in\mathbb{R}^p:||z||_2\leq B_z\}$, we consider  discriminators and generators as neural network models of the form:
\begin{align}
    D_\omega&:x\mapsto \sigma\left(\mathbf{w}_k^\mathsf{T}r_{k-1}(\mathbf{W}_{d-1}r_{k-2}(\dots r_1(\mathbf{W}_1(x)))\right)\,  \label{eqn:disc-model}\\
    G_\theta&:z\mapsto \mathbf{V}_ls_{l-1}(\mathbf{V}_{l-1}s_{l-2}(\dots s_1(\mathbf{V}_1z))),
\end{align}
where, $\mathbf{w}_k$ is a parameter vector of the output layer; for $i\in[1:k-1]$ and $j\in[1:l]$, $\mathbf{W}_i$ and $\mathbf{V}_j$ are parameter matrices; $r_i(\cdot)$ and $s_j(\cdot)$ are entry-wise activation functions of layers $i$ and $j$, i.e., for $\mathbf{a}\in\mathbb{R}^t$, $r_i(\mathbf{a})=\left[r_i(a_1),\dots,r_i(a_t)\right]$ and $s_i(\mathbf{a})=\left[s_i(a_1),\dots,s_i(a_t)\right]$; and $\sigma(\cdot)$ is the sigmoid function given by $\sigma(p)=1/(1+\mathrm{e}^{-p})$ (note that $\sigma$ does not appear in the discriminator in \cite[Equation~(7)]{JiZL21} as the discriminator considered in the neural net distance is not a soft classifier mapping to $[0,1]$). We assume that each $r_i(\cdot)$ and $s_j(\cdot)$ are $R_i$- and $S_j$-Lipschitz, respectively, and also that they are positive homogeneous, i.e., $r_i(\lambda p)=\lambda r_i(p)$ and $s_j(\lambda p)=\lambda s_j(p)$, for any $\lambda\geq 0$ and $p\in\mathbb{R}$. Finally, as modelled in \cite{neyshabur2015norm,salimans2016weight,golowich2018size,JiZL21}, we assume that the Frobenius norms of the parameter matrices are bounded, i.e., $||\mathbf{W}_i||_F\leq M_i$, $i\in[1:k-1]$, $||\mathbf{w}_k||_2\leq M_k$, and $||\mathbf{V}_j||_F\leq N_j$, $j\in[1:l]$. 
%this is redundant: Let $\hat{\theta}^*$ be \noteTS{GK: ``a global optimizer''?} the optimizer in \eqref{eqn:training-empirical}. 

We define the estimation error for a CPE loss GAN as 
\begin{align}\label{eqn:estimation-error-def}
    d^{(\ell)}_{\mathcal{F}_{nn}}(P_r,\hat{P}_{G_{\hat{\theta}^*}})-\inf_{\theta\in\Theta} d^{(\ell)}_{\mathcal{F}_{nn}}(P_r,P_{G_{\theta}}),
\end{align}
where $\hat{\theta}^*$ is the minimizer of \eqref{eqn:training-empirical} and present the following upper bound on the error.
%\noteTS{GK, should $Z_{i}$ use $j$ instead?}
% \balance
\begin{theorem}\label{thm:estimationerror-upperbound}
% In the setting described above, additionally assume the following.
% \begin{itemize}
%     \item The activation functions $r_i(\cdot)$, $i\in[1:k-1]$ and $s_j(\cdot)$, $j\in[1:l-1]$ are positive homogeneous, i.e., $r_i(\lambda p)=\lambda r_i(p)$ and $s_j(\lambda p)=\lambda s_j(p)$, for any $\lambda\geq 0$ and $p\in\mathbb{R}$,
%     \item The functions $\phi(\cdot)$ and $\psi(\cdot)$ are $L_\phi$- and $L_\psi$-Lipschitz, respectively. 
% \end{itemize}
In the setting described above, additionally assume that the functions $\phi(\cdot)$ and $\psi(\cdot)$ are $L_\phi$- and $L_\psi$-Lipschitz, respectively.
Then, with probability at least $1-2\delta$ over the randomness of training samples $S_x=\{X_i\}_{i=1}^n$ and $S_z=\{Z_j\}_{j=1}^m$, we have
%\begin{subequations}
\begin{align}
    &d^{(\ell)}_{\mathcal{F}_{nn}}(P_r,\hat{P}_{G_{\hat{\theta}^*}})-\inf_{\theta\in\Theta} d^{(\ell)}_{\mathcal{F}_{nn}}(P_r,P_{G_{\theta}})\nonumber\\
    &\leq \frac{L_\phi B_xU_\omega\sqrt{3k}}{\sqrt{n}}+\frac{L_\psi U_\omega U_\theta B_z\sqrt{3(k+l-1)}}{\sqrt{m}}\nonumber\\
    &\hspace{12pt}+U_\omega\sqrt{\log{\frac{1}{\delta}}}\left(\frac{L_\phi B_x}{\sqrt{2n}}+\frac{L_\psi B_zU_\theta}{\sqrt{2m}}\right), \label{eq:estimationboundrhs2}
\end{align}
%\end{subequations}
where the parameters $U_\omega\coloneqq M_k\prod_{i=1}^{k-1}(M_iR_i)$ and $U_\theta\coloneqq N_l\prod_{j=1}^{l-1}(N_jS_j)$.

In particular, when this bound is specialized to the case of $\alpha$-GAN by letting $\phi(p)=\psi(1-p)=\frac{\alpha}{\alpha-1}\left(1-p^{\frac{\alpha-1}{\alpha}}\right)$, the resulting bound is nearly identical to the terms in the RHS of~\eqref{eq:estimationboundrhs2}, except for substitutions $L_\phi \leftarrow 4C_{Q_x}(\alpha)$ and $L_\psi \leftarrow 4C_{Q_z}(\alpha)$, where $Q_x\coloneqq U_\omega B_x$, $Q_z\coloneqq U_\omega U_\theta B_z$, and
\begin{align} \label{eq:clipalpha}
    C_h(\alpha)\coloneqq\begin{cases}\sigma(h)\sigma(-h)^{\frac{\alpha-1}{\alpha}}, \ &\alpha\in(0,1]\\
    \left(\frac{\alpha-1}{2\alpha-1}\right)^{\frac{\alpha-1}{\alpha}}\frac{\alpha}{2\alpha-1}, &\alpha\in[1,\infty).
    \end{cases}
\end{align}
\end{theorem}
\begin{proof}[Proof sketch]\let\qed\relax
%A detailed proof is in the extended version. We present a proof sketch here.
Our proof involves the following steps:
\begin{itemize}[leftmargin=*]
    \item Building upon the proof techniques of Ji~\emph{et al.} \cite[Theorem~1]{JiZL21}, we bound the estimation error in terms of Rademacher complexities of \emph{compositional} function classes involving the CPE loss function. 
    \item We then upper bound these Rademacher complexities leveraging a contraction lemma for Lipschitz loss functions~\cite[Lemma~26.9]{shalev2014understanding}. We remark that this differs considerably from the way the bounds on Rademacher complexities in \cite[Corollary~1]{JiZL21} are obtained because of the explicit role of the loss function in our setting.  
    \item For the case of $\alpha$-GAN, we extend a result by Sypherd \emph{et al.}~\cite{sypherd2021journal} where they showed that $\alpha$-loss is Lipschitz for a logistic model with~\eqref{eq:clipalpha}. Noting that similar to the logistic model, we also have a sigmoid in the outer layer of the discriminator, we generalize the preceding observation by proving that $\alpha$-loss is Lipschitz when the input is equal to a sigmoid function acting on a \textit{neural network} model. This is the reason behind the dependence of the Lipschitz constant on the neural network model parameters (in terms of $Q_x$ and $Q_z$). Note that~\eqref{eq:clipalpha} is monotonically decreasing in $\alpha$, indicating the bound saturates. However, one is not able to make definitive statements regarding the estimation bounds for relative values of $\alpha$ because the LHS in~\eqref{eq:estimationboundrhs2} is \textit{also} a function of $\alpha$. Proof details are in 
    \if \extended 0%
    \cite[Appendix~E]{ISITextenver}.
    \fi%
    \if \extended 1%
    Appendix~\ref{proofoftheorem3}.
    \fi
\end{itemize}
\end{proof}

\begin{figure*}[t!]
    \centering
    \includegraphics[width=.9\linewidth]{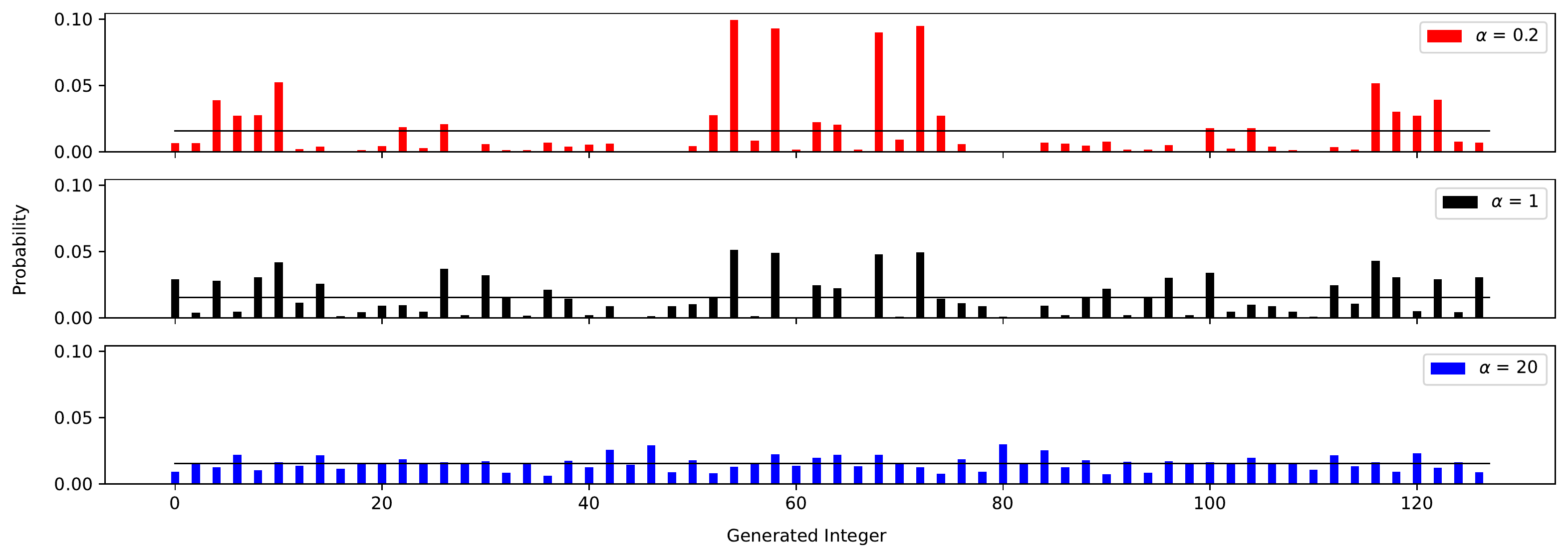}
    \caption{Histograms of averaged (over $10$ runs) $\hat{P}_{G_{\theta}}$ for $\alpha \in \{0.2, 1, 20\}$ in the \textbf{Base} setting. The thick black lines correspond to $P_{r}$ (uniform over the evens $\in [0,126]$). %, \emph{i.e.}, the uniform distribution over the even integers from 0 to 126.  %$\alpha = 1$ is the vanilla GAN. 
    No odd integers were output for any $\alpha$, indicating sufficient training.}
    %For $\alpha = 0.2$, there is far more variability in the modes, and $\alpha = 20$ skews most towards $P_r$.}
    \label{fig:simpleganhistograms}
\end{figure*}

\begin{table*}
    \centering
\begin{center}
\renewcommand{\arraystretch}{1.2}
\renewcommand{\tabcolsep}{2pt}
\begin{tabular}{|c|c!{\color{gray!50}\vrule }c!{\color{gray!50}\vrule }c!{\color{gray!50}\vrule }c|c!{\color{gray!50}\vrule }c!{\color{gray!50}\vrule }c!{\color{gray!50}\vrule }c|c!{\color{gray!50}\vrule }c!{\color{gray!50}\vrule }c!{\color{gray!50}\vrule }c|c!{\color{gray!50}\vrule }c!{\color{gray!50}\vrule }c!{\color{gray!50}\vrule }c|c!{\color{gray!50}\vrule }c!{\color{gray!50}\vrule }c!{\color{gray!50}\vrule }c|}
 \cline{2-21}
   \multicolumn{1}{c|}{}& \multicolumn{4}{c|}{\textbf{\%Noise = 0}} & \multicolumn{4}{c|}{\textbf{\%Noise = 10}} & \multicolumn{4}{c|}{\textbf{\%Noise = 15}} & \multicolumn{4}{c|}{\textbf{\%Noise = 20}} & \multicolumn{4}{c|}{\textbf{\%Noise = 30}}\\ 
 \hline
 $\pmb{\alpha}$ & \#Modes & \%Odd & TVD & JSD & \#Modes & \%Odd & TVD & JSD & \#Modes & \%Odd & TVD & JSD & \#Modes & \%Odd & TVD & JSD & \#Modes & \%Odd & TVD & JSD \\
 \hline
 0.2   
 & 55.6	& 0.0 & 0.618 &	0.272
 & 58.5 & 0.3 & 0.585 & 0.252
 & 56.6	& 0.0 & 0.635 & 0.283
 & 55.1 & 6.2 & 0.592 & 0.270
 & 56.0 & 27.5 & 0.663 & 0.338 \\
 \arrayrulecolor{gray!50}\hline \arrayrulecolor{black!100}
 0.5   
 & 59.6 & 0.0 &	0.586 &	0.249
 & 57.3	& 0.0 & 0.652 &	0.296
 & 58.6	& 0.0 &	0.566 & 0.236
 & 56.4	& 0.0 & 0.627 & 0.280
 & 61.0 & 0.0 &	0.561 &	0.224 \\
  \arrayrulecolor{gray!50}\hline \arrayrulecolor{black!100}
 0.7   
 & 60.7 & 0.0 &	0.597 & 0.255
 & 56.6	& 0.0 &	0.661 & 0.310
 & 56.7 & 0.0 &	0.618 & 0.279
 & 56.4 & 0.0 & 0.664 & 0.307
 & 55.4 & 0.0 & 0.653 & 0.297 \\
 \arrayrulecolor{gray!50}\hline\arrayrulecolor{black!100}
 1
 & 58.7 & 0.0 &	0.631 & 0.283
 & 60.3 & 0.0 &	0.582 & 0.247
 & 58.7 & 0.0 & 0.620 & 0.272
 & 58.6 & 0.0 &	0.609 & 0.271
 & 58.9 & 0.0 &	0.603 & 0.262 \\
 \arrayrulecolor{gray!50}\hline\arrayrulecolor{black!100}
4   
& 58.3 & 0.0 & 0.608 & 0.273   
& 58.3 & 1.2 & 0.618 & 0.282
& 57.6 & 4.5 & 0.650 & 0.300
& 58.2 & 1.5 & 0.596 & 0.265
& 61 & 0.0 & 0.591 & 0.250 \\
  \arrayrulecolor{gray!50}\hline \arrayrulecolor{black!100}
 10  
 & 61.8 & 0.0 &	0.478 & 0.174
 & 59.9 & 3.8 & 0.480 & 0.191
 & 62.3 & 10.6 & 0.503 & 0.202
 & 61.8 & 13.7 & 0.508 & 0.206
 & 61.8 & 14.1 & 0.486 & 0.199 \\
 \arrayrulecolor{gray!50}\hline\arrayrulecolor{black!100}
  20   
  & 63.2 & 0.0 & 0.327 & 0.088 
  & 62.7 & 5.3 & 0.328 & 0.103
  & 63.2 & 7.6 & 0.318 & 0.100
  & 63.5 & 9.0 & 0.299 & 0.098
  & 63.5 & 14.8 & 0.332 & 0.121 \\
 \hline
\end{tabular}
\end{center}
    \caption{Results for $\alpha$-GAN on Toy Dataset in \textbf{base} and \textbf{noisy} settings}
    \label{tab:simplegannoisytrue}
    \vspace{-0.1in}
\end{table*}

\vspace{-0.1in}
\section{Experimental Results}
We now present experimental results of $\alpha$-GAN trained over the set of $\alpha \in [0.2, 20]$ for a simple dataset. The real training examples in this dataset consist of $25,600$ unsigned seven-bit binary representations of uniformly-drawn \emph{even} integers from $0$ to $126$; \emph{i.e.,} $P_{r}$ is the uniform distribution on even integers between $0$ and $126$.
Note that we sometimes refer to even integer(s) as mode(s) (as is common in GAN literature).

We consider two settings: the first is a standard GAN training setup (\textbf{Base}) and the second (\textbf{Noisy}) differs from the first \emph{only} in  introducing noisy real samples. %is identical to the first except for the an additional layer of complexity.
%Specifically, for the second setting, we simulate noisy \textit{real} training examples as follows: for a random subset of the real training examples, we flip the least significant bit (LSB) in the binary representation of the even integer in order to make it \textit{odd}.
This may resemble a practical scenario where, unbeknownst to the practitioner implementing a GAN, the training data is mislabeled, \emph{e.g.}, when a cat is labeled as a dog.
Nevertheless in both settings, the goal of the generator $G_{\theta}$ is to learn the real distribution $P_{r}$.
Overall, we find that $\alpha$-GAN exhibits interesting characteristics as a function of $\alpha$, and there is significant utility in tuning $\alpha$ away from $\alpha = 1$ (vanilla GAN).
For both cases, we consider the same architectures for the generator and discriminator as detailed below. Our implementation builds on~\cite{simplegan}; full experimental details (and further results) are 
\if \extended 0%
in~\cite{ISITextenver}.
\fi
\if \extended 1%
in Appendix~F.
\fi%
%\noindent\textbf{Base Experimental Setup.}
\noindent\textbf{Model and experimental details.}
%The GAN architecture consists of a generator with input and output 
The generator, with $7$-length input and output, is modeled as $G_\theta(z) = \sigma(W_g z+b_g)$, where $\theta = \{W_g, b_g\}$, $W_g \in \mathbb R^{7 \times 7}$, $b_g \in \mathbb R^7$, and $ \sigma: \mathbb R \to (0,1)$ is the sigmoid function; 
%given by $\sigma(t) = (1 + e^{-t})^{-1}$
the discriminator takes a $7$-length input and outputs a scalar with $ D_\omega(x) = \sigma(W_d x+b_d)$, where $\omega = \{W_d, b_d\}$, $W_d \in \mathbb R^{1 \times 7}$, and $b_d \in \mathbb R$. 
We use the following hyperparameter settings, which are fixed for all $\alpha$: learning rate of $0.001$, the Adam optimizer~\cite{kingma2014adam}, standard normal noise \emph{i.e.}, $P_{Z} = \mathcal{N}(\mathbf{0},\mathbb{I}_{7})$, batch size of $256$ for both the real and generator noise samples, and $2,000$ training epochs. % to allow each $\alpha$-GAN to train until convergence.
After training, % each $\alpha$-GAN is trained, 
we feed each trained $\alpha$-GAN generator the same set of 20,000 noise samples (also from $P_{Z}$) to evaluate its performance\footnote{To eliminate additional randomness from test data, we use the same 20k samples, thereby illustrating the performance variations from changing $\alpha$.}. All results are averaged over $10$ runs for each $\alpha$, where the GAN is retrained in each run.

\noindent\textbf{Noisy real data setup.} %(\noteTS{not done})
%\noteTS{MW, please write about about how the randomization is done to make the noisy real examples}
We simulate noisy \textit{real} training examples as follows: for a chosen percentage ($\%$\textbf{Noise}) of corrupt samples that are sampled uniformly from the real training examples, we flip the least significant bit (LSB) in the binary representation of the even integer in order to make it \textit{odd}. For every $\%\textbf{Noise}\in \{10,15,20,30\}$, we train an $\alpha$-GAN with the corresponding noisy real samples.

%To create the noisy real training examples, we uniformly choose a percentage of the real training examples and flip the LSB of their binary representations. 
\noindent\textbf{Evaluation metrics.}
 We evaluate the performance of the \textbf{Base} and \textbf{Noisy} cases using the following four metrics: number of output modes, percentage of synthetic outputs that are odd, and both TVD and JSD between the empirical $\hat{P}_{G_\theta}$ and the uniform $P_r$. 
The number of output modes refers to the number of unique even integers between $0$ and $126$ (maximum of $64$) output by the generator. 
We present these metrics for $\alpha \in \{0.2, 0.5, 0.7, 1, 4, 10, 20\}$ in Table~\ref{tab:simplegannoisytrue}. Figure~\ref{fig:simpleganhistograms} illustrates the (averaged) output probability distributions $\hat{P}_{G_{\theta}}$ for $\alpha = \{0.2,1,20\}$; note that $\alpha = 20$ exhibits the best overall performance, as it yields an averaged distribution closest to $P_{r}$.

\noindent\textbf{Interpretation of results in Figure~\ref{fig:simpleganhistograms} and Table~\ref{tab:simplegannoisytrue}.}
\begin{enumerate}[wide, labelindent=0pt]
    \item Our results suggest that for each $\alpha$, $\alpha$-GAN learns a mixture of Gaussians\footnote{We conjecture this because the latent noise driving the generator is Gaussian.} with the mixture approaching the uniform distribution for larger $\alpha$. The results in Fig.~\ref{fig:simpleganhistograms} confirm a conjecture raised in~\cite[Figure~2]{KurriSS21} that while different choices of $\alpha$ may ideally have equivalent convergence (now proved in Thm.~\ref{thm:equivalenceinconvergence}), in practice, there will be significant differences in the output distributions for each $\alpha$ arising from how the different gradients for $\alpha$-GAN affect convergence. Increasing the number of epochs did not change the observed behavior. However, as illustrated by the plots in Appendix~F\if \extended 1%
    ,
    \fi%
    \if \extended 0%
    in \cite{ISITextenver},
    \fi%
    larger $\alpha$ require more epochs to converge than smaller $\alpha$, an observation supported by the fact that $\alpha$-loss is non-convex for $\alpha > 1$ and becomes more non-convex as $\alpha$ increases.
    %increasing non-convexity of $\alpha$-loss as $\alpha$ increases.  % (cf. \cite{ISITextenver}).%, yielding distinct performance.
    \item The result for $\alpha = 20$ in Fig.~\ref{fig:simpleganhistograms} is perhaps best explained by Sypherd \emph{et al.}~\cite{sypherd2021journal} where they show that in the standard supervised classification setting, models trained with $\alpha > 1$  approach the average probability of error (in estimating all the modes), thus yielding better overall performance. 
    In the GAN setting, we know from~\cite{KurriSS21} that as $\alpha \rightarrow \infty$, $\alpha$-GAN approaches the TV GAN. Evidently in this scenario, GANs resembling TV GAN far outperform the vanilla GAN.
    \item %\noteTS{not done} The results in Table~\ref{tab:simplegannoisytrue} for the \textbf{Noisy} setting are interesting: we find that holistically, performance improves both as $\alpha$ increases \emph{and} decreases from $1$ (vanilla GAN).
    %Intuitively from work by~\cite{sypherd2021journal}, $\alpha < 1$ is penalizing and $\alpha > 1$ is accuracy. optimal generator is more sensitive for smaller $\alpha$ (ITW), whereas the larger $\alpha$, they don't put as much emphasis on the noisy samples, then, the generator is ``okay'' to generate odd integers. When the discriminator sees the noise samples, it weights them heavily and in turn the generator knows not to output those.
    The results in Table~\ref{tab:simplegannoisytrue} indicate that larger $\alpha$ perform the best with respect to average TVD and JSD. However, with increasing noise, larger $\alpha$ values also lead to more odd integer outputs; in other words, larger $\alpha$ learn the noisy distribution better. %These results align with the results in \cite{PACGANLin} where the authors show that TV GAN is more robust to mode collapse (i.e., captures all modes where now we have additional odd modes in the data). 
    %On the other hand, small $\alpha$ are more robust to the odd modes relative to large $\alpha$ while still performing better than $\alpha = 1$ in terms of average TVD and JSD. 
    In summary, several questions yet remain on evaluating the role of $\alpha$ in learning from noisy data. 
\end{enumerate}

\section{Conclusion}
Building on our prior work introducing $\alpha$-GANs, we have introduced three new results here on the one-to-one correspondence between CPE losses and $f$-divergences, convergence properties of the Arimoto divergences induced by $\alpha$-GANs, and the estimation error for CPE loss GANs including $\alpha$-GAN. Our results on a toy dataset suggest that tuning $\alpha$ can enhance the quality of the synthetic data, in this case, with larger values of $\alpha$ offering more accuracy with respect to the real distribution. 
More work is needed to better understand the choice of $\alpha$ in limiting mode collapse. Our recent work suggests that tuning $\alpha<1$ improves classification accuracy for imbalanced datasets~\cite{sypherd2021journal}; we conjecture this will hold for $\alpha$-GANs when the real data has an imbalance in samples for different modes.
We believe the analysis here can help guide how $\alpha$-GANs can address these challenges rigorously. 

\if \extended 1
\appendices
\section{CPE Loss-Based GANs: Additional Observations}\label{appendix:details-omitted-lossfn}
Let $\phi(\cdot)\coloneqq-\ell(1,\cdot)$ and $\psi(\cdot)\coloneqq-\ell(0,\cdot)$ in the sequel. The functions $\phi$ and $\psi$ are assumed to be monotonically increasing and decreasing functions, respectively, so as to retain the intuitive interpretation of the vanilla GAN (that the discriminator should output high values to real samples and low values to the generated samples). These functions should also satisfy the constraint
\begin{align}\label{eqn:condnonfnsforGAN}
    \phi(t)+\psi(t)\leq \phi({1}/{2})+\psi({1}/{2}),\ \text{for all}\ t\in[0,1], 
\end{align}
so that the optimal discriminator guesses uniformly at random (i.e., outputs a constant value ${1}/{2}$ irrespective of the input) when $P_r=P_{G_\theta}$. A loss function $\ell(y,\hat{y})$ is said to be \emph{symmetric}~\cite{reid2010composite} if $\psi(t)=\phi(1-t)$, for all $t\in[0,1]$. Notice that the value function considered by Arora \emph{et al.}~\cite{AroraGLMZ17} is a special case of $\eqref{eqn:lossfnps1}$, i.e., $\eqref{eqn:lossfnps1}$ recovers the value function in \cite[Equation~(2)]{AroraGLMZ17} when the loss function $\ell(y,\hat{y})$ is symmetric. For symmetric losses, concavity of the function $\phi$ is a sufficient condition for satisfying \eqref{eqn:condnonfnsforGAN}, but not a necessary condition.
\section{Equivalence of the Jensen-Shannon Divergence and the Total Variation Distance}\label{appnedix:simpler}
We first show that the total variation distance is stronger than the Jensen-Shannon divergence, i.e., $D_{\text{TV}}(P_n\|P)\rightarrow 0$ as $n\rightarrow \infty$ implies $D_{\text{JS}}(P_n\|P)\rightarrow 0$ as $n\rightarrow \infty$. Suppose $D_{\text{TV}}(P_n||P)\rightarrow 0$ as $n\rightarrow \infty$. Using the fact that the total variation distance upper bounds the Jensen-Shannon divergence~\cite[Theorem 3]{Lin91}, we have $D_{\text{JS}}(P_n||P)\leq (\log_\mathrm{e}{2}) D_{\text{TV}}(P_n||P)$, for each $n\in\mathbbm{N}$. This implies that $D_{\text{JS}}(P_n||P)\rightarrow 0$ as $n\rightarrow \infty$ since $D_{\text{TV}}(P_n||P)\rightarrow 0$ as $n\rightarrow \infty$. The proof for the other direction, i.e., the Jensen-Shannon divergence is stronger than the total variation distance, is exactly along the same lines as that of \cite[Theorem~2(1)]{ArjovskyCB17} using triangle and Pinsker's inequalities.

\section{Proof of Theorem~\ref{thm:correspondence}}\label{apndx:proof-of-thm1}
Consider a symmetric CPE loss $\ell(y,\hat{y})$, i.e., $\ell(1,\hat{y})=\ell(0,1-\hat{y})$. We may define an associated margin-based loss using a bijective link function $l:\mathbb{R}\rightarrow [0,1]$ as
\begin{align}\label{eqn:margin-from-CPE}
    \tilde{\ell}(t):=\ell(1,l(t)),
\end{align}
where the link $l$ satisfies a mild regularity condition 
\begin{align}\label{eqn:regularityonlink}
l(-t)=1-l(t)
\end{align}
(e.g., sigmoid function, $\sigma(t)=1/(1+\mathrm{e}^{-t})$ satisfies this condition). Consider the inner optimization problem in \eqref{eq:Goodfellowobj} with the value function in \eqref{eqn:lossfnps1} for this CPE loss $\ell$.
\begin{align}
    &\sup_\omega\int_{\mathcal{X}}(-p_r(x)\ell(1,D_\omega(x))-p_{G_\theta}(x)\ell(0,D_\omega(x)))\ dx\nonumber\\
    &= \int_{\mathcal{X}}\sup_{p_x\in[0,1]}(-p_r(x)\ell(1,p_x)-p_{G_\theta}(x)\ell(0,p_x))\ dx\label{eqn:thm1proof6}\\
   &=\int_{\mathcal{X}}\sup_{p_x\in[0,1]}(-p_r(x)\ell(1,p_x)-p_{G_\theta}(x)\ell(1,1-p_x))\ dx\label{eqn:thm1proof1}\\
    &=\int_{\mathcal{X}}\sup_{t_x\in\mathbb{R}}(-p_r(x)\ell(1,l(t_x))-p_{G_\theta}(x)\ell(1,1-l(t_x)))dx\\
        &=\int_{\mathcal{X}}\sup_{t_x\in\mathbb{R}}(-p_r(x)\ell(1,l(t_x))-p_{G_\theta}(x)\ell(1,l(-t_x)))\ dx\label{eqn:thm1proof2}\\
    &=\int_{\mathcal{X}}\sup_{t_x\in\mathbb{R}}(-p_r(x)\tilde{\ell}(t_x)-p_{G_\theta}(x)\tilde{\ell}(-t_x)\ dx\label{eqn:thm1proof3}\\
   & =\int_{\mathcal{X}}p_{G_\theta}(x)\left(-\inf_{t_x\in\mathbb{R}}\left(\tilde{l}(-t_x)+\frac{p_r(x)}{p_{G_\theta}(x)}\tilde{l}(t_x)\right)\right) dx\label{eqn:eqn:tm1proof4}
\end{align}
where \eqref{eqn:thm1proof1} follows because the CPE loss $\ell(y,\hat{y})$ is symmetric, \eqref{eqn:thm1proof2} follows from \eqref{eqn:regularityonlink}, and \eqref{eqn:thm1proof3} follows from the definition of the margin-based loss $\tilde{\ell}$ in \eqref{eqn:margin-from-CPE}. Now note that the function $f$ defined as
\begin{align}\label{eqn:tm1proof5}
    f(u)=-\inf_{t\in\mathbb{R}}\left(\tilde{\ell}(-t)+u\tilde{\ell}(t)\right)
\end{align}
is convex since the infimum of affine functions is concave (observed earlier in \cite{NguyenWJ09} in a correspondence between margin-based loss functions and $f$-divergences). So, from \eqref{eqn:eqn:tm1proof4}, we get
\begin{align}
    \sup_\omega\int_{\mathcal{X}}&(-p_r(x)\ell(1,D_\omega(x))-p_{G_\theta}(x)\ell(0,D_\omega(x)))\ dx\nonumber\\
    &=\int_{\mathcal{X}}p_{G_\theta}(x)f\left(\frac{p_r(x)}{p_{G_\theta}(x)}\right)\ dx\\
    &=D_f(P_r\|P_{G_\theta}).
\end{align}
Thus, the resulting min-max optimization in \eqref{eqn:GANgeneral} reduces to minimizing the $f$-divergence, $D_f(P_r\|P_{G_\theta})$ with $f$ as given in \eqref{eqn:tm1proof5}.

For the converse statement, first note that given a symmetric $f$-divergence, it follows from \cite[Theorem~1(b) and Corollary~3]{NguyenWJ09} that there exists a margin-based loss function $\tilde{\ell}$ such that $f$ can be expressed in the form
\eqref{eqn:tm1proof5}. We may define an associated symmetric CPE loss $\ell(y,\hat{y})$ with
\begin{align}
\ell(1,\hat{y}):=\tilde{\ell}(l^{-1}(\hat{y})),
\end{align}
where $l^{-1}$ is the inverse of the same link function. Now repeating the steps as in $\eqref{eqn:thm1proof6}-\eqref{eqn:eqn:tm1proof4}$, it is clear that the GAN based on this (symmetric) CPE loss results in minimizing the same symmetric $f$-divergence. 

 \section{Proof of Theorem~\ref{thm:equivalenceinconvergence}}\label{proofoftheorem4}
Noticing that $D_{f_\infty}(\cdot\|\cdot)=D_{\text{TV}}(\cdot\|\cdot)$ (see \cite{osterreicher2003new}, \cite[Theorem~2]{KurriSS21}), it suffices to show that $D_{f_\alpha}(\cdot\|\cdot)$ is equivalent to $D_{\text{TV}}(\cdot\|\cdot)$, for $\alpha>0$, i.e., $D_{f_\alpha}(P_n||P)\rightarrow 0$ as $n\rightarrow \infty$ if and only if $D_{\text{TV}}(P_n||P)\rightarrow 0$ as $n\rightarrow \infty$. To this end, we employ a property of the Arimoto divergence $D_{f_\alpha}$ which gives lower and upper bounds on it in terms of the total variation distance, $D_{\text{TV}}$. In particular, \"{O}sterreicher and Vajda~\cite[Theorem~2]{osterreicher2003new} proved that for any $\alpha>0$, probability distributions $P$ and $Q$, we have
\begin{align}\label{eqn:boundsonArimoto}
    \gamma_\alpha(D_{\text{TV}}(P||Q))\leq D_{f_\alpha}(P||Q)\leq \gamma_\alpha(1)D_{\text{TV}}(P||Q),
\end{align} 
where the function $\gamma_\alpha:[0,1]\rightarrow \mathcal{R}$ defined by $\gamma_\alpha(p)=\frac{\alpha}{\alpha-1}\left(\left(\left(1+p\right)^\alpha+\left(1-p\right)^\alpha\right)^\frac{1}{\alpha}-2^{\frac{1}{\alpha}}\right)$ for $\alpha\in(0,1)\cup(1,\infty)$ is convex and strictly monotone increasing such that $\gamma_\alpha(0)=0$ and $\gamma_\alpha(1)=\frac{\alpha}{\alpha-1}\left(2-2^\frac{1}{\alpha}\right)$. 

We first prove the `only if' part, i.e., $D_{f_\alpha}(P_n||P)\rightarrow 0$ as $n\rightarrow \infty$ implies $D_{\text{TV}}(P_n||P)\rightarrow 0$ as $n\rightarrow \infty$. Suppose $D_{f_\alpha}(P_n||P)\rightarrow 0$. From the lower bound in \eqref{eqn:boundsonArimoto}, it follows that $\gamma_\alpha(D_{\text{TV}}(P_n||P))\leq D_{f_\alpha}(P_n||P)$, for each $n\in\mathbbm{N}$. This implies that $\gamma_\alpha(D_{\text{TV}}(P_n||P))\rightarrow 0$ as $n\rightarrow \infty$. We show below that $\gamma_\alpha$ is invertible and $\gamma_\alpha^{-1}$ is continuous. Then it would follow that $\gamma_\alpha^{-1}\gamma_\alpha(D_{\text{TV}}(P_n||P))=D_{\text{TV}}(P_n||P)\rightarrow \gamma_\alpha^{-1}(0)=0$ as $n\rightarrow \infty$ proving that Arimoto divergence is stronger than the total variation distance. It remains to show that $\gamma_\alpha$ is invertible and $\gamma_\alpha^{-1}$ is continuous. Invertibility follows directly from the fact that $\gamma_\alpha$ is strictly monotone increasing function. For the continuity of $\gamma_\alpha^{-1}$, it suffices to show that $\gamma_\alpha(C)$ is closed for a closed set $C\subseteq [0,1]$. The closed set $C$ is compact since a closed subset of a compact set ($[0,1]$ in this case) is also compact. Note that convexity of $\gamma_\alpha$ implies continuity and $\gamma_\alpha(C)$ is compact since a continuous function of a compact set is also compact. By Heine-Borel theorem, this gives that $\gamma_\alpha(C)$ is closed (and bounded) as desired.

We prove the `if part' now, i.e., $D_{\text{TV}}(P_n||P)\rightarrow 0$ as $n\rightarrow \infty$ implies $D_{f_\alpha}(P_n||P)\rightarrow 0$. It follows from the upper bound in $\eqref{eqn:boundsonArimoto}$ that $D_{f_\alpha}(P_n||P)\leq D_{\text{TV}}(P_n||P)$, for each $n\in\mathbbm{N}$. This implies that $D_{f_\alpha}(P_n||P)\rightarrow 0$ as $n\rightarrow \infty$ which completes the proof. 
\section{Proof of Theorem~\ref{thm:estimationerror-upperbound}}\label{proofoftheorem3}
We upper bound the estimation error in terms of the Rademacher complexities of appropriately defined \emph{compositional} classes building upon the proof techniques of \cite[Theorem~1]{JiZL21}. We then bound these Rademacher complexities using a contraction lemma~\cite[Lemma~26.9]{shalev2014understanding}. Details are in order.

We first review the notion of Rademacher complexity.
\begin{definition}[Rademacher complexity]
Let $\mathcal{G}_\Omega:=\{g_\omega: g_\omega\ \text{is a function from}\ \mathcal{X}\ \text{to}\ \mathbb{R}$, $\omega\in\Omega\}$ and $S=\{X_1.\dots,X_n\}$ be a set of random samples in $\mathcal{X}$ drawn independent and identically distributed (i.i.d.) from a distribution $P_X$. Then, the Rademacher complexity of $\mathcal{G}_\Omega$ is defined as
\begin{align}
    \mathcal{R}_S(\mathcal{G}_\Omega)=\mathbb{E}_{X,\epsilon}\sup_{\omega\in\Omega}\left\lvert\frac{1}{n}\sum_{i=1}^n\epsilon_ig_\omega(x_i)\right\rvert
\end{align}
where $\epsilon_1,\dots,\epsilon_n$ are independent random variables uniformly distributed on $\{-1,+1\}$. 
\end{definition}
We write our discriminator model in \eqref{eqn:disc-model} in the form
\begin{align}\label{eqn:thm3proof1}
    D_\omega(x)=\sigma(f_\omega(x)),
\end{align}
where $f_\omega$ is exactly the same discriminator model defined in \cite[Equation~(26)]{JiZL21}. Now by following the similar steps as in \cite[Equations~(16)-(18)]{JiZL21} by replacing $f_\omega(\cdot)$ in the first and second expectation terms in the definition of $d_{\mathcal{F}_{nn}}(\cdot,\cdot)$ by $\phi(D_\omega(\cdot))$ and $-\psi(D_\omega(\cdot))$, respectively, we get
\begin{align}
    &d^{(\ell)}_{\mathcal{F}_{nn}}(P_r,\hat{P}_{G_{\hat{\theta}^*}})-\inf_{\theta\in\Theta} d^{(\ell)}_{\mathcal{F}_{nn}}(P_r,P_{G_{\theta}})\nonumber\\
    &\leq 2\sup_{\omega}\left\lvert\mathbb{E}_{X\sim P_r}\phi(D_\omega(X))-\frac{1}{n}\sum_{i=1}^n\phi(D_\omega(X_i))\right\rvert\nonumber\\
    &\hspace{12pt}+2\sup_{\omega,\theta}\left\lvert\mathbb{E}_{Z\sim P_Z}\psi(D_\omega(g_\theta(Z)))-\frac{1}{m}\sum_{j=1}^m\psi(D_\omega(g_\theta(Z_j)))\right\rvert\label{eqn:thm3proof2}
\end{align}
Let us denote the supremums in the first and second terms in \eqref{eqn:thm3proof2} by $F^{(\phi)}(X_1,\dots,X_n)$ and $G^{(\psi)}(Z_1,\dots,Z_m)$, respectively. We next bound $G^{(\psi)}(Z_1,\dots,Z_m)$. Note that $\psi(\sigma(\cdot))$ is $\frac{L_\psi}{4}$-Lipschitz since it is a composition of two Lipschitz functions $\psi(\cdot)$ and $\sigma(\cdot)$ which are $L_\psi$- and $\frac{1}{4}$-Lipschitz respectively. For any $z_1,\dots,z_j,\dots,z_m,z_j^\prime$, using $\sup_r|h_1(r)|-\sup_r|h_2(r)|\leq \sup_r |h_1(r)-h_2(r)|$, we have 
\begin{align}
   &G^{(\psi)}(z_1,\dots,z_j,\dots,z_m)-G^{(\psi)}(z_1,\dots,z_j^\prime,\dots,z_m)\nonumber\\
   &\leq \sup_{\omega,\theta}\frac{1}{m}\left\lvert\psi(D_\omega(g_\theta(z_j)))-\psi(D_\omega(g_\theta(z_j^\prime)))\right\rvert\\
   &\leq \sup_{\omega,\theta}\frac{1}{m}\left\lvert\psi(\sigma(f_\omega(g_\theta(z_j))))-\psi(\sigma(f_\omega(g_\theta(z_j^\prime))))\right\rvert\label{eqn:thm3proof3}\\
      &\leq \frac{L_\psi}{4}\sup_{\omega,\theta}\frac{1}{m}\left\lvert \sigma(f_\omega(g_\theta(z_j)))-\sigma(f_\omega(g_\theta(z_j^\prime)))\right\rvert\label{eqn:thm3proof4}\\
            &\leq \frac{L_\psi}{4}\frac{2}{m}\left(M_k\prod_{i=1}^{k-1}(M_iR_i)\right)\left(N_l\prod_{j=1}^{l-1}(N_jS_j)\right)B_z\label{eqn:thm3proof6}\\
            &=\frac{L_\psi Q_z}{2m}\label{eqn:thm3proof7},
\end{align}
where \eqref{eqn:thm3proof3} follows from \eqref{eqn:thm3proof1}, \eqref{eqn:thm3proof4} follows because $\psi(\sigma(\cdot))$ is $\frac{L_\psi}{4}$-Lipschitz, \eqref{eqn:thm3proof6} follows by using the Cauchy-Schwarz inequality and the fact that $||Ax||_2\leq ||A||_F|||x||_2$ (as observed in \cite{JiZL21}), and \eqref{eqn:thm3proof7} follows by defining
\begin{align}\label{eqn:Q_zparameter}
    Q_z\coloneqq\left(M_k\prod_{i=1}^{k-1}(M_iR_i)\right)\left(N_l\prod_{j=1}^{l-1}(N_jS_j)\right)B_z.
\end{align}
Using \eqref{eqn:thm3proof7}, the McDiarmid's inequality~\cite[Lemma~26.4]{shalev2014understanding} implies that, with probability at least $1-\delta$,
\begin{align}
   &G^{(\psi)}(Z_1,\dots,Z_j,\dots,Z_m)\nonumber\\
  % &\leq\mathbb{E}_ZG^{(\psi)}(Z_1,\dots,Z_i,\dots,Z_m)+\frac{L_\psi Q_z}{2m}\sqrt{\frac{m}{2}\log{\frac{1}{\delta}}}\\
   &\leq\mathbb{E}_ZG^{(\psi)}(Z_1,\dots,Z_j,\dots,Z_m)+\frac{L_\psi Q_z}{2}\sqrt{\log{\frac{1}{\delta}}/(2m)}.\label{eqn:thm3proof8}
\end{align}
Following the standard steps similar to \cite[Equation~(20)]{JiZL21}, the expectation term in \eqref{eqn:thm3proof8} can be upper bounded as
\begin{align}
  \mathbb{E}_ZG^{(\psi)}&(Z_1,\dots,Z_j,\dots,Z_m)\nonumber\\
  &\leq 2\mathbb{E}_{Z,\epsilon}\sup_{\omega,\theta}\left\lvert\frac{1}{m}\sum_{j=1}^m\epsilon_j\psi(D_\omega(g_\theta(Z_j)))\right\lvert \\
  &=:2\mathcal{R}_{S_z}(\mathcal{H}^{(\psi)}_{\Omega\times\Theta})
\end{align}
So, we have, with probability at least $1-\delta$,
\begin{align}
    G^{(\psi)}(Z_1,\dots,Z_j,\dots,Z_m)
    \leq 2\mathcal{R}_{S_z}(\mathcal{H}^{(\psi)}_{\Omega\times\Theta})+\sqrt{\log{\frac{1}{\delta}}}\frac{L_\psi Q_z}{2\sqrt{2m}}.\label{eqn:thm3proof9}
\end{align}
Using a similar approach, we have, with probability at least $1-\delta$,
\begin{align}\label{eqn:thm3proof10}
    F^{(\phi)}(X_1,\dots,X_n)\leq 2\mathcal{R}_{S_x}(\mathcal{F}_\Omega^{(\phi)})+\sqrt{\log{\frac{1}{\delta}}}\frac{L_\phi Q_x}{2\sqrt{2n}},
\end{align}
where 
\begin{align}
    \mathcal{R}_{S_x}(\mathcal{F}_\Omega^{(\phi)}):=\mathbb{E}_{X,\epsilon}\sup_{\omega}\left\lvert\frac{1}{n}\sum_{i=1}^n\epsilon_i\phi(D_\omega(X_i))\right\lvert. 
\end{align}
Combining \eqref{eqn:thm3proof2}, \eqref{eqn:thm3proof9}, and \eqref{eqn:thm3proof10} using a union bound, we get, with probability at least $1-2\delta$,
\begin{align}
    d^{(\ell)}_{\mathcal{F}_{nn}}(P_r,\hat{P}_{G_{\hat{\theta}^*}})-\inf_{\theta\in\Theta} &d^{(\ell)}_{\mathcal{F}_{nn}}(P_r,P_{G_{\theta}})\nonumber\\
    &\leq 4\mathcal{R}_{S_x}(\mathcal{F}_\Omega^{(\phi)})+4\mathcal{R}_{S_z}(\mathcal{H}^{(\psi)}_{\Omega\times\Theta})\nonumber\\
    &\hspace{12pt}+\sqrt{\log{\frac{1}{\delta}}}\left(\frac{L_\phi Q_x}{\sqrt{2n}}+\frac{L_\psi Q_z}{\sqrt{2m}}\right)\label{eqn:thm3proof11}. 
\end{align}
Now we bound the Rademacher complexities in the RHS of \eqref{eqn:thm3proof11}. We present the contraction lemma on Rademacher complexity required to obtain these bounds. For $A\subset\mathbb{R}^n$, let $\mathcal{R}(A):=\mathbb{E}_\epsilon\left[\sup_{a\in A}\left\lvert\frac{1}{n}\sum_{i=1}^n\epsilon_ia_i\right\rvert\right]$.
\begin{lemma}[Lemma~26.9, \cite{shalev2014understanding}]\label{lemma:contraction}
For each $i\in\{1,\dots,n\}$, let $\gamma_i:\mathbb{R}\rightarrow\mathbb{R}$ be a $\rho$-Lipschitz function. Then, for $A\subset\mathbb{R}^n$,
\begin{align}
    \mathcal{R}(\gamma\circ A)\leq \rho \mathcal{R}(A),
\end{align}
where $\gamma\circ A:=\{(\gamma_1(a_1),\dots,\gamma_n(a_n)):a\in A\}$.
\end{lemma}
Note that $\phi(\sigma(\cdot))$ is $\frac{L_\phi}{4}$-Lipschitz since it is a composition of two Lipschitz functions $\phi(\cdot)$ and $\sigma(\cdot)$ which are $L_\phi$- and $\frac{1}{4}$-Lipschitz respectively. Consider
 \begin{align}
       &\mathcal{R}_{S_x}(\mathcal{F}_\Omega^{(\phi)})\nonumber\\
       &= \mathbb{E}_X\left[\mathcal{R}\left(\{\left(\phi(D_\omega(X_1)),\dots,\phi(D_\omega(X_n))\right):\omega\in\Omega\}\right)\right]\\
       &= \mathbb{E}_X\left[\mathcal{R}\left(\{\left(\phi(\sigma(f_\omega(X_1))),\dots,\phi(\sigma(f_\omega(X_n)))\right):\omega\in\Omega\}\right)\right]\label{eqn:thm3proof12}\\
       &\leq \frac{L_\phi}{4}\mathbb{E}_X\left[\mathcal{R}\left(\{\left(f_\omega(X_1),\dots,(f_\omega(X_n)\right):\omega\in\Omega\}\right)\right]\label{eqn:thm3proof13}\\
       &\leq \frac{L_\phi Q_x\sqrt{3k}}{4\sqrt{n}}\label{eqn:thm3proof14}
    \end{align}
    where \eqref{eqn:thm3proof12} follows from \eqref{eqn:thm3proof1}, \eqref{eqn:thm3proof13} follows from Lemma~\ref{lemma:contraction} by substituting $\gamma(\cdot)=\phi(\sigma(\cdot))$, and \eqref{eqn:thm3proof14} follows from \cite[Proof of Corollary~1]{JiZL21}. Using a similar approach,  we obtain
    \begin{align}
        \mathcal{R}_{S_z}(\mathcal{H}^{(\psi)}_{\Omega\times\Theta})\leq \frac{L_\psi Q_z\sqrt{3(k+l-1)}}{4\sqrt{m}}\label{eqn:thm3proof15}.
    \end{align}
Substituting \eqref{eqn:thm3proof14} and \eqref{eqn:thm3proof15} into \eqref{eqn:thm3proof11} gives \eqref{eq:estimationboundrhs2}.     
\subsection{Specialization to $\alpha$-GAN}
Let $\phi_\alpha(p)=\psi_\alpha(1-p)=\frac{\alpha}{\alpha-1}\left(1-p^{\frac{\alpha-1}{\alpha}}\right)$. It is shown in \cite[Lemma~7]{sypherd2021journal} that $\phi_\alpha(\sigma(\cdot))$ is  $C_h(\alpha)$-Lipschitz in $[-h,h]$, for $h>0$, with $C_h(\alpha)$ as given in $\eqref{eq:clipalpha}$. Now using the Cauchy-Schwarz inequality and the fact that $||Ax||_2\leq ||A||_F||x||_2$, it follows that 
\begin{align}
    |f_\omega(\cdot)|\leq Q_x,\\
    |f_\omega(g_\theta(\cdot))|\leq Q_z,
\end{align}
where $Q_x:=M_k\prod_{i=1}^{k-1}(M_iR_i)B_x$ and with $Q_z$ as in \eqref{eqn:Q_zparameter}. So, we have $f_\omega(\cdot)\in[-Q_x,Q_x]$ and $f_\omega(g_\theta(\cdot))\in[-Q_z,Q_z]$. Thus, we have that $\psi_\alpha(\sigma(\cdot))$ and $\phi_\alpha(\sigma(\cdot))$ are $C_{Q_z}(\alpha)$- and $C_{Q_x}(\alpha)$-Lipschitz, respectively. Now specializing the steps \eqref{eqn:thm3proof4} and \eqref{eqn:thm3proof13} with these Lipschitz constants, we get the following bound with the substitutions $\frac{L_\phi}{4} \leftarrow C_{Q_x}(\alpha)$ and $\frac{L_\psi}{4} \leftarrow 4C_{Q_z}(\alpha)$ in \eqref{eq:estimationboundrhs2}:
\begin{align}
    &d^{(\ell_\alpha)}_{\mathcal{F}_{nn}}(P_r,\hat{P}_{G_{\hat{\theta}^*}})-\inf_{\theta\in\Theta} d^{(\ell_\alpha)}_{\mathcal{F}_{nn}}(P_r,P_{G_{\theta}})\nonumber\\
    &\leq \frac{4C_{Q_x}(\alpha) Q_x\sqrt{3k}}{\sqrt{n}}+\frac{4C_{Q_z}(\alpha) Q_z\sqrt{3(k+l-1)}}{\sqrt{m}}\nonumber\\
    &\hspace{12pt}+2\sqrt{2\log{\frac{1}{\delta}}}\left(\frac{C_{Q_x}(\alpha)Q_x}{\sqrt{n}}+\frac{C_{Q_z}(\alpha)Q_z}{\sqrt{m}}\right).
\end{align}

\section{Further Experimental Details}
The GAN architecture is as follows: the generator, with $7$-length input and output, is modeled as $G_\theta(z) = \sigma(W_g z+b_g)$, where $\theta = \{W_g, b_g\}$, $W_g \in \mathbb R^{7 \times 7}$, $b_g \in \mathbb R^7$, and $ \sigma: \mathbb R \to (0,1)$ is the sigmoid function given by $\sigma(t) = (1 + e^{-t})^{-1}$;
the discriminator takes a $7$-length input and outputs a scalar with $ D_\omega(x) = \sigma(W_d x+b_d)$, where $\omega = \{W_d, b_d\}$, $W_d \in \mathbb R^{1 \times 7}$, and $b_d \in \mathbb R$. 

In order to convert the generator's output $G_\theta(z) \in (0,1)^7$ into the corresponding binary representation $ b \in \{0,1\}^7$ when evaluating the performance of the trained generator, we use a threshold, i.e. for $i \in \{1,2,\dots,7\}$, $b_i = 1$ if $G_\theta(z)_i \ge 0.5$ and $b_i = 0$ otherwise.

Training was done on a computing cluster using NVIDIA V100 GPUs. See Figs.~\ref{fig:simplegan-disc-output_alpha-0.5}, \ref{fig:simplegan-disc-output_alpha-0.7}, \ref{fig:simplegan-disc-output_alpha-1}, \ref{fig:simplegan-disc-output_alpha-4}, \ref{fig:simplegan-disc-output_alpha-10}, and \ref{fig:simplegan-disc-output_alpha-20} for the plots of the discriminator output (of a single run) for the real, generated, and validation data for $\alpha = 0.5$, $0.7$, $1$, $4$, $10$, and $20$, respectively, in the \textbf{Base} setting. The validation data consists of $5,000$ synthetic examples created in the same way as the real training data. Note that the discriminator output converges close to $1/2$ for all $\alpha$.

\begin{figure*}[b!]
    \centering
    \includegraphics[width=0.95\linewidth]{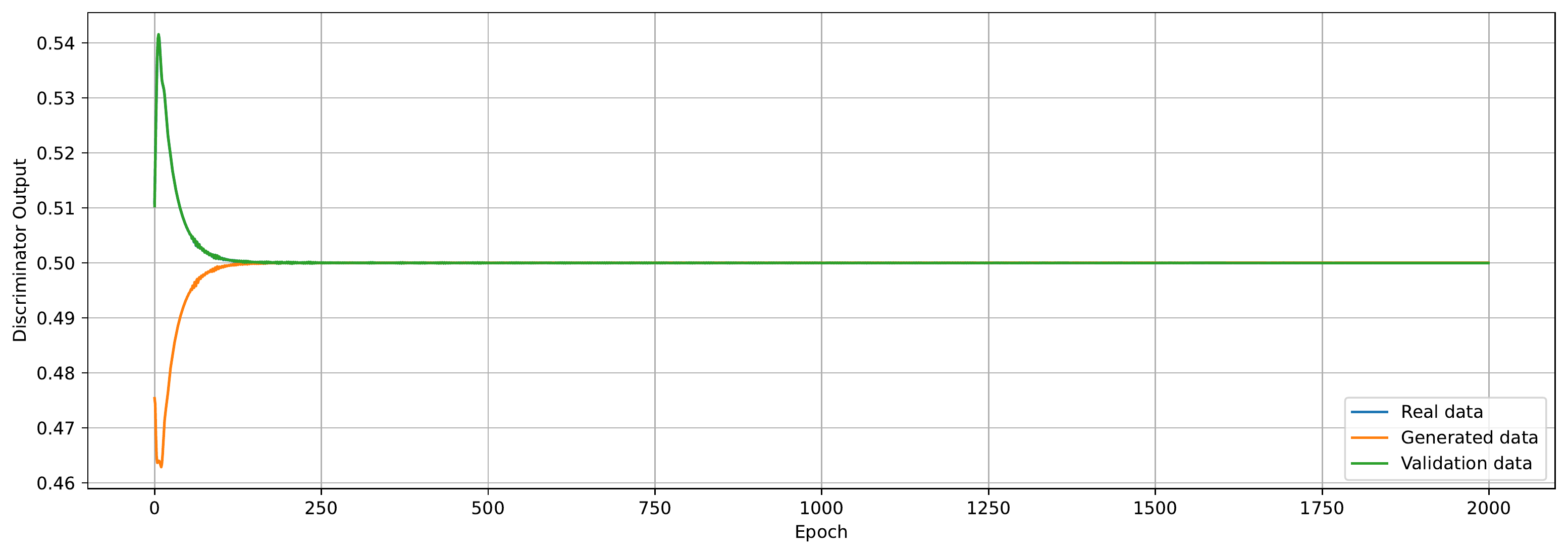}
    \caption{Plot of discriminator output for real, generated, and validation data over $2,000$ epochs for $\alpha = 0.5$ in the \textbf{Base} setting.
    }
    \label{fig:simplegan-disc-output_alpha-0.5}
\end{figure*}

\begin{figure*}[b!]
    \centering
    \includegraphics[width=0.95\linewidth]{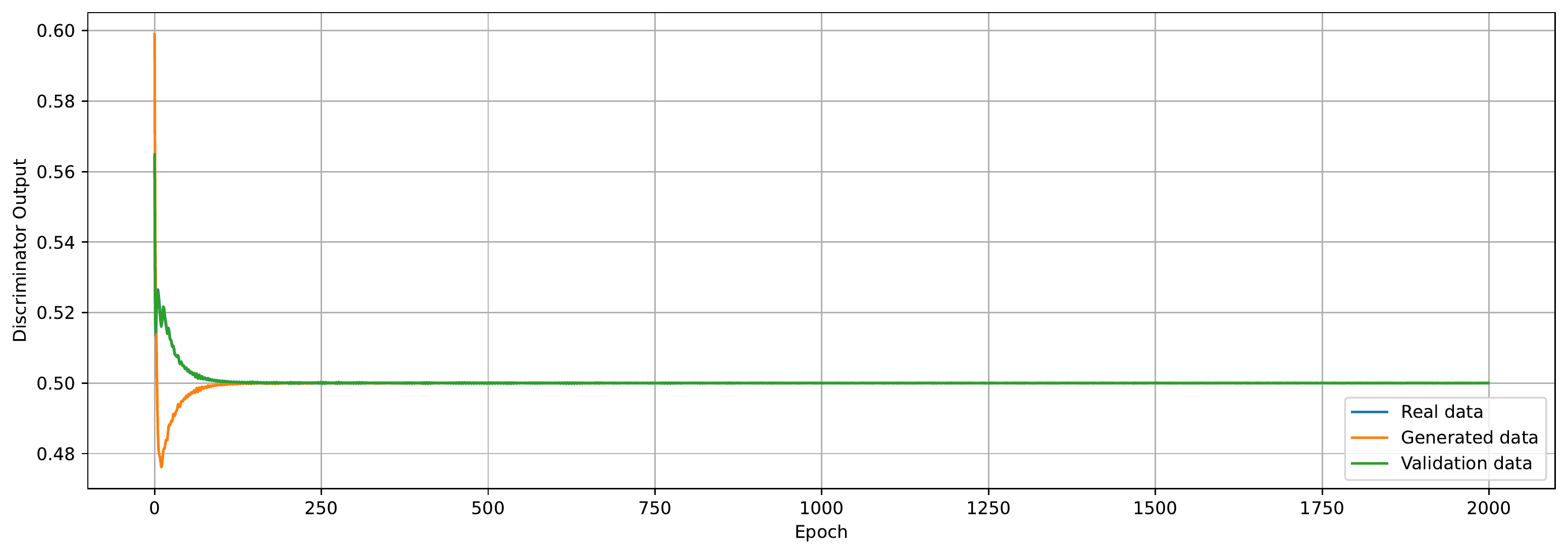}
    \caption{Plot of discriminator output for real, generated, and validation data over $2,000$ epochs for $\alpha = 0.7$ in the \textbf{Base} setting.
    }
    \label{fig:simplegan-disc-output_alpha-0.7}
\end{figure*}

\begin{figure*}[htbp]
    \centering
    \includegraphics[width=0.95\linewidth]{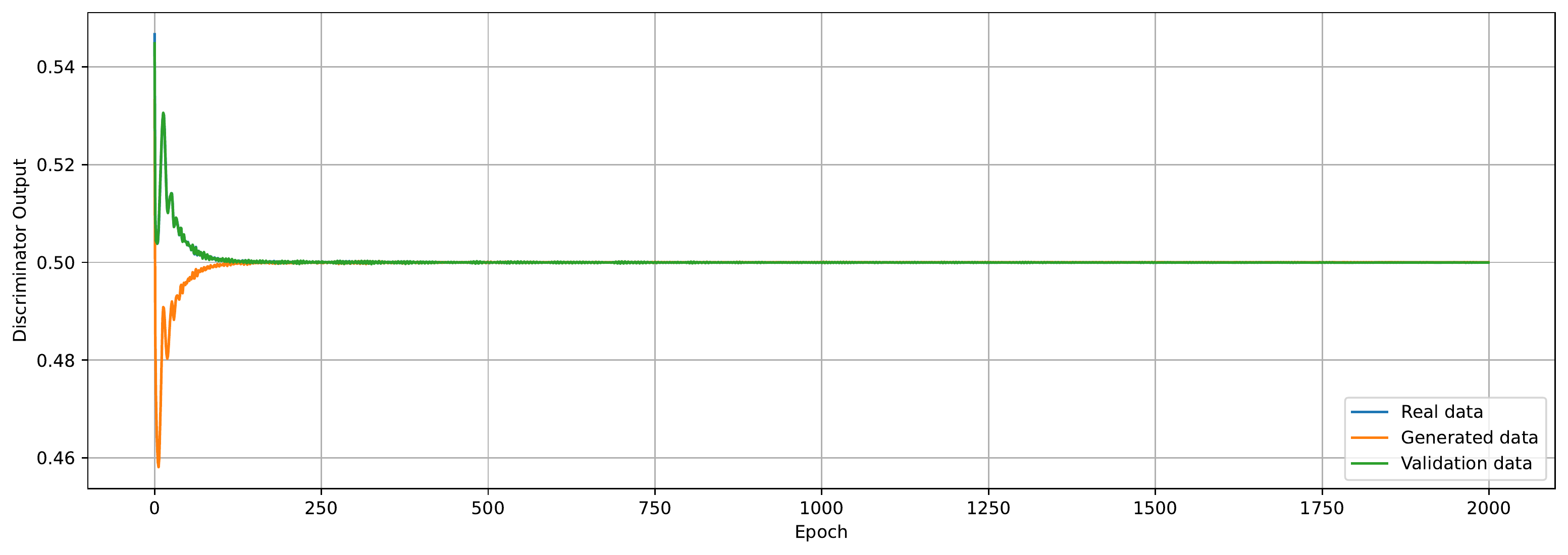}
    \caption{Plot of discriminator output for real, generated, and validation data over $2,000$ epochs for $\alpha = 1$ in the \textbf{Base} setting.
    }
    \label{fig:simplegan-disc-output_alpha-1}
\end{figure*}

\begin{figure*}[htbp]
    \centering
    \includegraphics[width=0.95\linewidth]{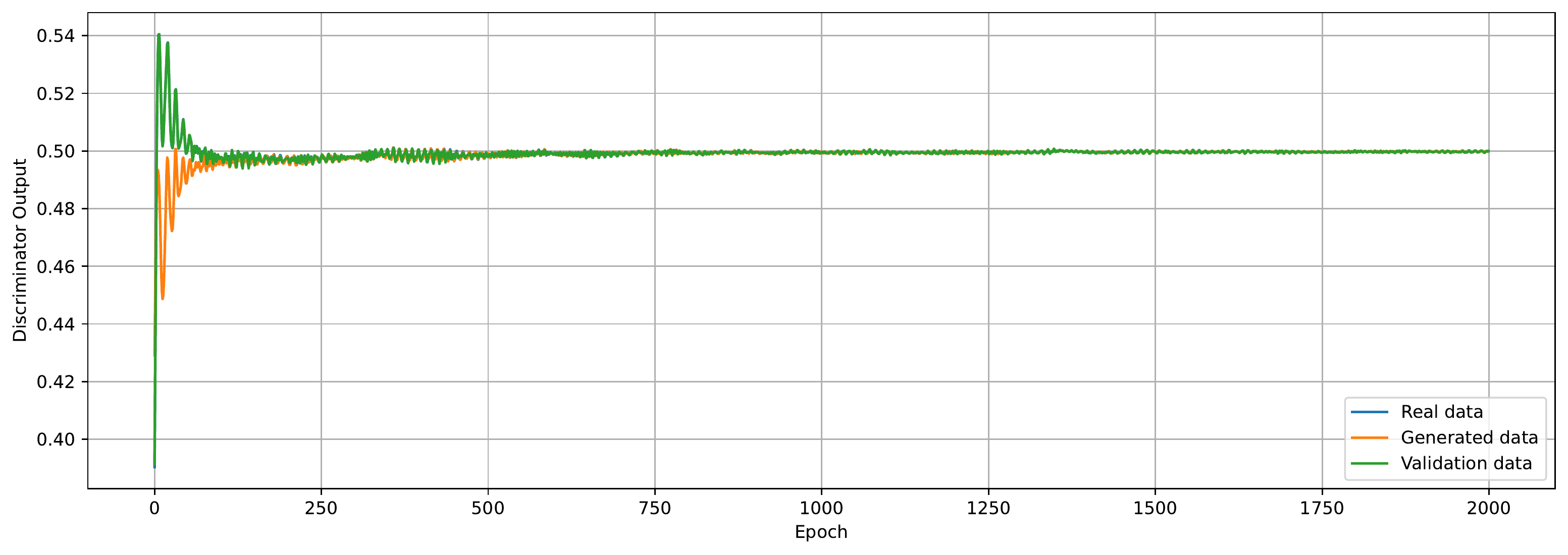}
    \caption{Plot of discriminator output for real, generated, and validation data over $2,000$ epochs for $\alpha = 4$ in the \textbf{Base} setting.
    }
    \label{fig:simplegan-disc-output_alpha-4}
\end{figure*}

\begin{figure*}[htbp]
    \centering
    \includegraphics[width=0.95\linewidth]{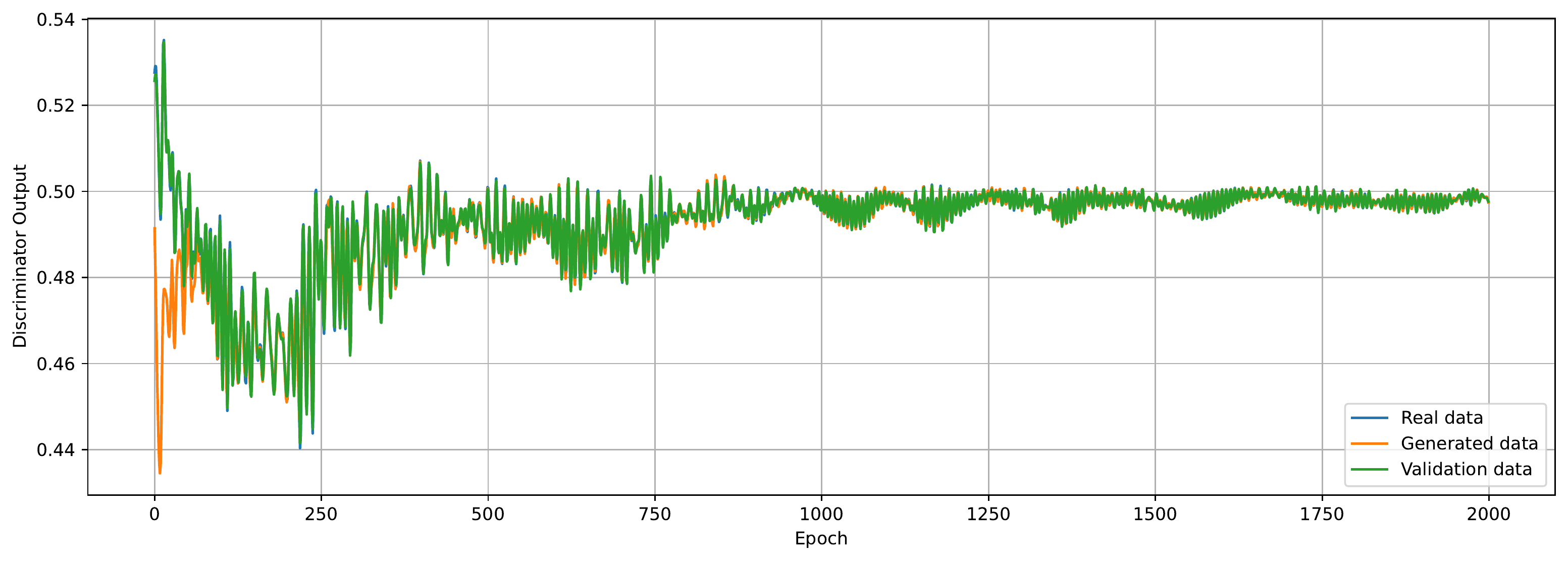}
    \caption{Plot of discriminator output for real, generated, and validation data over $2,000$ epochs for $\alpha = 10$ in the \textbf{Base} setting.
    }
    \label{fig:simplegan-disc-output_alpha-10}
\end{figure*}

\begin{figure*}[htbp]
    \centering
    \includegraphics[width=0.95\linewidth]{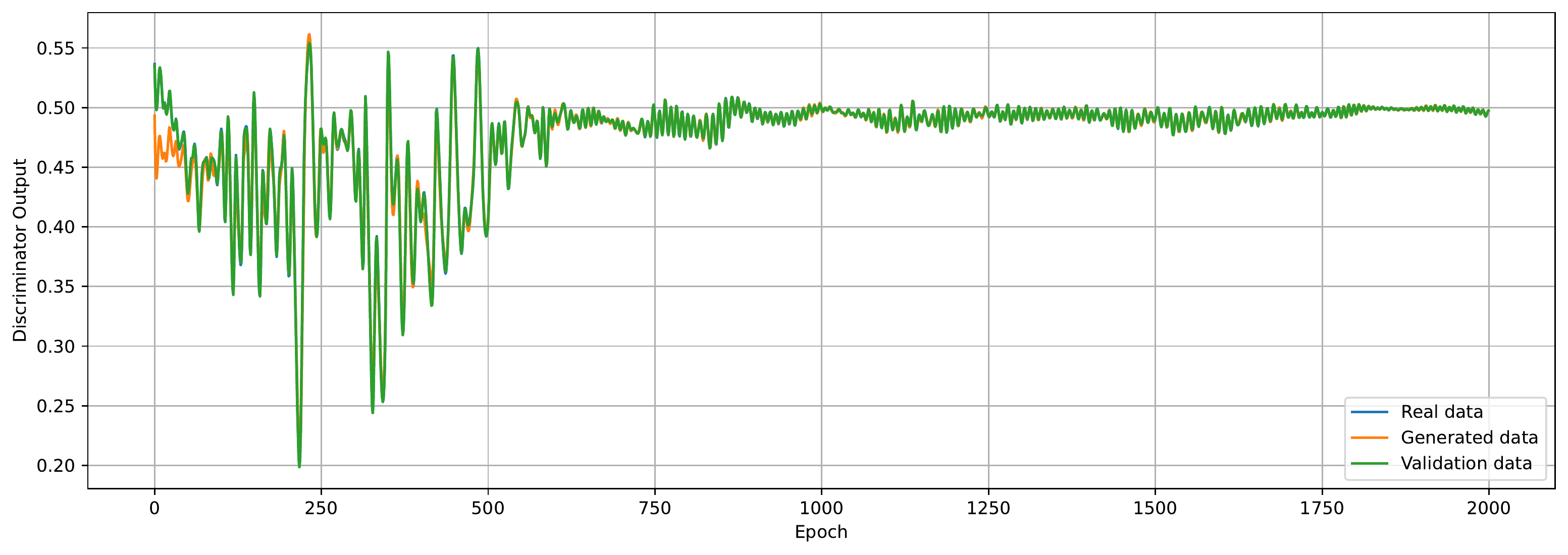}
    \caption{Plot of discriminator output for real, generated, and validation data over $2,000$ epochs for $\alpha = 20$ in the \textbf{Base} setting.
    }
    \label{fig:simplegan-disc-output_alpha-20}
\end{figure*}
\fi
%\balance
\bibliographystyle{IEEEtran}
\clearpage
\bibliography{Bibliography}
\end{document}